%% file: main.tex
\documentclass{article}

\usepackage[preprint]{neurips_2025}

\usepackage[utf8]{inputenc} 
\usepackage[T1]{fontenc}    
\usepackage{hyperref}       
\usepackage{url}            
\usepackage{booktabs}       
\usepackage{booktabs}
\usepackage{graphicx}       
\usepackage{amsfonts}       
\usepackage{nicefrac}       
\usepackage{microtype}      
\usepackage{xcolor}         

\usepackage{amsmath,amssymb,amsthm}

\newtheorem{theorem}{Theorem}

\newtheorem{proposition}[theorem]{Proposition}

\usepackage{algorithm}
\usepackage{algpseudocode}
\usepackage{enumitem}
\title{CKT-WAM: Parameter-Efficient Context Knowledge Transfer Between World Action Models} 

\author{
  Yuhua Jiang$^{1,2}$, Yijun Guo$^{2}$, Hongbing Yang$^{1,2}$, Guojun Lei$^{2}$, Nuo Chen$^{1}$, Yinuo Zhang$^{1}$, \\ 
  \textbf{Shaoqiang Yan$^{1}$, Bo Lin$^{1}$, Feifei Gao$^{1}$, Biqing Qi$^{3}$} \\
  $^{1}$ Tsinghua University \quad 
  $^{2}$ LivsynRobotics  \quad 
  $^{3}$ Shanghai AI Laboratory
}

\begin{document}

\maketitle

\begin{abstract}
World action models (WAMs) provide a powerful generative framework for embodied control, yet transferring knowledge across heterogeneous WAMs remains challenging due to mismatched latent interfaces, high adaptation cost, and the rigidity of conventional distillation objectives. We propose \textbf{CKT-WAM}, a parameter-efficient \textbf{C}ontext \textbf{K}nowledge \textbf{T}ransfer framework that transfers teacher WAM's knowledge into a student WAM through a compact context in the text embedding space, rather than output imitation or dense hidden-state matching.
Specifically, CKT-WAM extracts intermediate teacher hidden states, reduces the number of tokens via compressors' learnable-query cross attention (LQCA), and transforms them through an always-on generalized adapter, a lightweight router, and sparsely activated specialized adapters.
The resulting context is then appended to the student's conditioning textual embeddings, thereby injecting the transferred knowledge into the student with minimal architectural modification.
Experiments show that CKT-WAM consistently improves zero-shot generalization and achieves the best overall performance on LIBERO-Plus, reaching 86.1\% total success rate with only 1.17\% trainable parameters, while approaching full fine-tuning performance. Beyond simulation, CKT-WAM also demonstrates strong real-world long-horizon manipulation ability, achieving the best average success rate of 83.3\% across four multi-step and long-horizon tasks.
Code is available at \url{https://github.com/YuhuaJiang2002/CKT-WAM}.

\end{abstract}

\section{Introduction}
\label{sec:intro}

World action models (WAMs) have emerged as a powerful paradigm for embodied intelligence, explicitly modeling how visual observations evolve under action to support purposeful behavior \citep{ye2026world}. 
Unlike standard vision-language-action (VLA) models \citep{brohan2023rt1,zitkovich2023rt2,asyncVLA}, which primarily learn a direct, reactive mapping from visual and linguistic inputs to robotic actions, WAMs integrate a generative understanding of the physical world into the control loop. 
Building upon early latent dynamics research \citep{ha2018worldmodels,hafner2020dreamer}, recent literature has formalized two dominant WAM architectures: \emph{imagine-then-execute} models that causally generate future visual trajectories before predicting actions \citep{li2026causal,feng2025vidar}, and \emph{joint-modeling} frameworks that denoise future video and action tokens within a shared generative process \citep{bi2025motus,zhu2025unified}. 
While these developments provide a holistic scaffolding for downstream decision-making, their explicit video synthesis often incurs substantial test-time latency. 
This bottleneck has prompted parallel efforts to bypass explicit video decoding during inference---such as VPP \citep{hu2024video}, UVA \citep{li2025unified}, and Fast-WAM \citep{yuan2026fastwam}, which decouples video co-training from test-time future imagination to achieve real-time efficiency. 
This tension between generative capability and deployment efficiency naturally raises a critical question: \textbf{rather than scaling a single model in isolation, can we efficiently transfer action-relevant world knowledge between different WAMs?}

This question is of immense practical importance. 
In real-world deployments, practitioners often have access to a powerful but computationally heavy teacher WAM and a lightweight student WAM bounded by strict latency or memory budgets. 
However, transferring knowledge across generative WAMs is non-trivial. 
While knowledge distillation has seen immense success in vision and language models \citep{hinton2015distilling,touvron2021training}, applying it across heterogeneous action architectures presents unique challenges. Direct fine-tuning is prohibitively expensive. Output-level imitation, such as policy distillation or logit matching \citep{rusu2016policy}, can be brittle when the teacher and student utilize disparate latent parameterizations or asymmetric action heads. 
Conversely, deep structural alignment via full hidden-state or attention-map matching \citep{romero2015fitnets,zagoruyko2017attention,wang2020minilm} imposes rigid architectural constraints and introduces severe overhead by forcing layer-by-layer mimicry \citep{jiao2020tinybert}. 
What is needed is a transfer mechanism that is both \emph{architecture-compatible} and \emph{parameter-efficient}.

In this paper, we propose \textbf{CKT-WAM}, a parameter-efficient \emph{context knowledge transfer} framework between WAMs. Our key idea is that knowledge need not be transferred through logits, actions, or full hidden-state supervision; instead, a teacher WAM can provide a compact \emph{context interface} that is directly consumable by the student's native generative pipeline. Given an observation and instruction, CKT-WAM extracts hidden states from an intermediate DiT layer of a teacher WAM. We propose compressors that reduce the token count via learnable-query cross attention (LQCA), and then transform the compressed teacher states through an always-on generalized adapter, a lightweight router, and sparsely activated specialized adapters. The resulting routed context tokens preserve task-relevant sufficient statistics while avoiding the cost of transferring full representations, and are injected into the conditioning textual embeddings of the student WAM with minimal architectural modification.
Experiments show that CKT-WAM consistently improves zero-shot generalization and achieves the best overall performance on LIBERO-Plus, reaching 86.1\% total success rate with only 1.17\% trainable parameters, while approaching full fine-tuning performance. Beyond simulation, CKT-WAM also demonstrates strong real-world long-horizon manipulation ability, achieving the best average success rate of 83.3\% across four multi-step tasks.


\section{Related Work}
\label{sec:related_work}

\paragraph{WAMs.}
World models were first developed for latent dynamics learning in prediction, planning, and control \citep{ha2018worldmodels,hafner2019planet,hafner2020dreamer,hafner2023dreamerv3}. 
Recent work extends them to embodied control, where WAMs model future visual dynamics under actions \citep{ye2026world}. 
Existing WAMs mainly follow two directions: \emph{imagine-then-execute} \citep{li2026causal,feng2025vidar} and \emph{joint modeling} \citep{bi2025motus,zhu2025unified}. 
Our work instead studies efficient knowledge transfer from a stronger teacher WAM to a lightweight student WAM.

\paragraph{Knowledge Transfer.}
Knowledge transfer has evolved from distillation to intermediate feature and attention transfer \citep{hinton2015distilling,romero2015fitnets,zagoruyko2017attention,jiao2020tinybert}, and has also been explored in policy learning \citep{rusu2016policy}. 
However, most methods assume aligned outputs or representations, which is less suitable for WAM-to-WAM transfer. 
CKT-WAM instead transfers compact intermediate context tokens through the student's native conditioning pathway, enabling flexible transfer across heterogeneous WAM pairs.

\paragraph{Parameter-Efficient Adapters.}
Parameter-efficient fine-tuning (PEFT) provides an efficient alternative to full fine-tuning. 
LoRA learns low-rank task updates \citep{hu2022lora}, while spectral variants such as PiSSA and KaSA highlight the role of initialization and singular-component selection \citep{meng2024pissa,wang2025kasa}. 
MoE-style PEFT further improves capacity through routed experts \citep{liu2024adamole,tian2024hydralora}. 

\section{Method}
\label{sec:method}

\begin{figure*}[t]
    \centering
    \includegraphics[width=\textwidth]{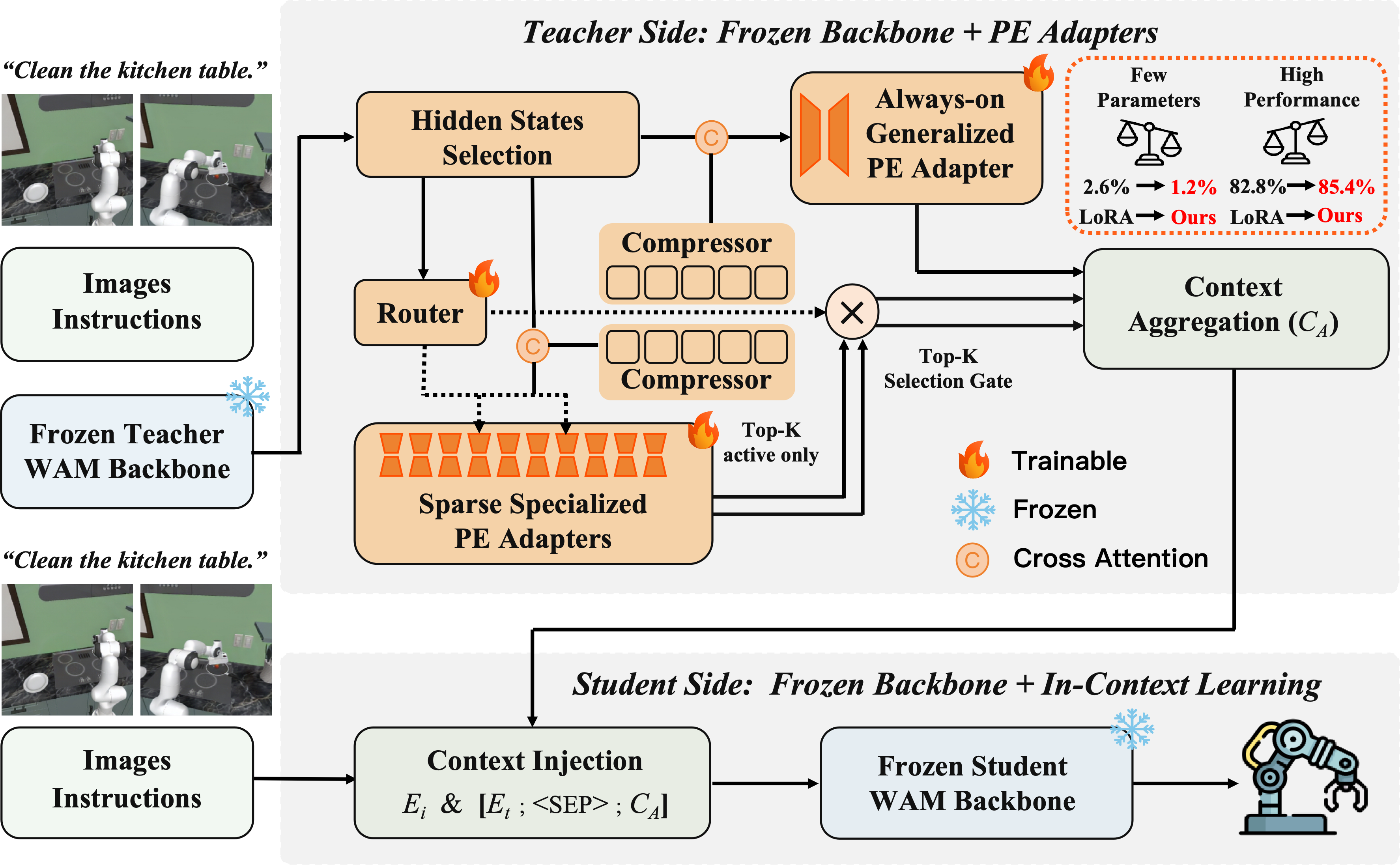}
\caption{Overview of \textbf{CKT-WAM}. Hidden states from a selected intermediate layer of the teacher WAM are transformed by the CKT module. 
The compressors use learnable-query cross-attention (LQCA) for token compression.
A routing module aggregates the relevant branches into a compact context $C_A$, which the student WAM concatenates with its original textual embeddings $E_t$. 
The input image embeddings are denoted by $E_i$, and PE is short for parameter-efficient.}
    \label{fig:ckt_wam}
\end{figure*}

\subsection{Overview}
We propose {CKT-WAM}, a parameter-efficient context knowledge transfer framework between a teacher WAM and a student WAM. 
Rather than transferring full representations, CKT-WAM transfers task-relevant sufficient knowledge through two stages: teacher-side context construction and student-side context injection. 
The teacher produces an action-relevant intermediate representation, which is compressed by LQCA into a fixed-length context $C_A$ and processed by an always-on generalized adapter, a lightweight router, and sparsely activated specialized adapters. 
The resulting context is appended to the student's conditioning textual tokens $E_t$, injecting teacher knowledge while leaving the student denoising backbone unchanged. 
Both backbones are frozen, and only the lightweight CKT module is optimized.

\subsection{Teacher-Side Hidden States Selection for Efficient and Single-Pass Extraction}
Suppose the teacher WAM contains $L^{\mathrm{tea}}$ transformer blocks, and let
$H^{\mathrm{tea}}_{\ell} \in \mathbb{R}^{B \times N \times d^{\mathrm{tea}}}$
denote the hidden states after the $\ell$-th block, where $B$ is the batch size,
$N$ is the sequence length, and $d^{\mathrm{tea}}$ is the teacher hidden dimension.
We extract transferable knowledge from an intermediate block $\ell^* \in \{1,\dots,L^{\mathrm{tea}}\}$.
Using an intermediate teacher representation offers a better cost--utility trade-off than relying on the deepest layer, since it is already semantically structured and action-aware while avoiding the most expensive final-stage features.

To reduce both latency and memory overhead, we do not use the teacher WAM to explicitly predict future video or future actions during inference.
Instead, the teacher is repurposed as a \emph{single-pass observation encoder}:
its input contains only the observed image tokens, and its hidden states are used solely for one-shot context construction.
Let $X_{\mathrm{img}}$ and $X_{\mathrm{text}}$ denote the teacher-side visual and textual tokens, respectively.
We execute the teacher DiT once and obtain the intermediate hidden states of the teacher as: 
\begin{equation}
H_T = H^{\mathrm{tea}}_{\ell^*}(X_{\mathrm{img}},X_{\mathrm{text}}; t^* = 0),
\label{eq:teacher_once}
\end{equation}
where the timestep is fixed to $t^* = 0$.
The resulting hidden states are transformed into $C_A$ once and then reused across the entire denoising trajectory of the student WAM.
Hence, the student no longer requires teacher forwarding at every denoising step.
We set $t^* = 0$ because the teacher is not used to model a future denoising trajectory, but only to encode the current scene into a stable world-aware representation. 
Using the clean endpoint avoids injecting timestep-specific noise corruption into teacher features and yields representations that are more deterministic and easier to reuse across all student timesteps.
This is also consistent with recent one-shot teacher-querying strategies for efficient world-model-assisted action generation.

\subsection{Parameter-Efficient CKT Module}

\paragraph{Compressor: learnable-query cross-attention (LQCA) for token compression.}  
We first map $H_T$ into the student feature space through a shared bottleneck adapter:
\begin{equation}
Z
=
\mathrm{Drop}\left(
\mathrm{LN}\left(
\phi( H_T W_{\mathrm{down}} ) W_{\mathrm{up}}
\right)
\right)
\in
\mathbb{R}^{B \times N \times d^{\mathrm{stu}}},
\label{eq:shared_proj}
\end{equation}
where
$W_{\mathrm{down}} \in \mathbb{R}^{d^{\mathrm{tea}} \times d_b}$ and
$W_{\mathrm{up}} \in \mathbb{R}^{d_b \times d^{\mathrm{stu}}}$ are the down- and up-projection matrices,
$d^\mathrm{stu}$ is the dimension of the student's textual embeddings, 
$d_b$ is the bottleneck dimension,
$\phi(\cdot)$ is a nonlinear activation function,
$\mathrm{LN}(\cdot)$ denotes layer normalization,
and $\mathrm{Drop}(\cdot)$ denotes dropout.

The core role of the {compressor} is to reduce the teacher sequence from $N$ tokens to a much smaller set of transferable tokens through LQCA.
To this end, we introduce two learnable query banks,
$Q_g \in \mathbb{R}^{1 \times K_g \times d^{\mathrm{stu}}}$ and
$Q_s \in \mathbb{R}^{1 \times K_s \times d^{\mathrm{stu}}}$,
for the generalized and specialized branches, respectively, where typically
$K_g, K_s \ll N$.
Let $\tilde{Q}_g \in \mathbb{R}^{B \times K_g \times d^{\mathrm{stu}}}$ and
$\tilde{Q}_s \in \mathbb{R}^{B \times K_s \times d^{\mathrm{stu}}}$ be their batch-wise broadcast.
We derive keys and values from $Z$ as
$K_Z = Z W_K,
V_Z = Z W_V,$
where
$W_K, W_V \in \mathbb{R}^{d^{\mathrm{stu}} \times d^{\mathrm{stu}}}$ are learnable projections.
The compressed teacher contexts are then obtained by
\begin{align}
C_*^{0}
&=
\mathrm{LN}\!\left(
\mathrm{MHCA}(\tilde{Q}_*, K_Z, V_Z) + \tilde{Q}_*
\right)
\in
\mathbb{R}^{B \times K_* \times d^{\mathrm{stu}}}, \quad * \in \{g,s\},
\label{eq:compact_context}
\end{align}
where $*=g$ and $*=s$ correspond to the generalized and specialized branches, respectively, and $\mathrm{MHCA}(Q,K,V)$ denotes multi-head cross-attention.

This design can be viewed as a token-level compression operator: instead of passing the full teacher sequence of length $N$ to subsequent adapters, the compressor summarizes it into only $K_g$ or $K_s$ learnable query-conditioned tokens.
As a result, the module not only allows the generalized and specialized branches to extract different aspects of the teacher representation, but also substantially reduces the downstream transfer cost by replacing long teacher sequences with compact contexts.
The compressor enforces a bottleneck that extracts only task-relevant information from teacher representations.

\paragraph{Always-on generalized adapter.}
On top of $C_g^{0}$, we apply one always-on generalized residual bottleneck adapter:
\begin{equation}
C_g
=
C_g^{0}
+
\phi\big(
C_g^{0} W_{\mathrm{down}}^{(g)} 
\big) W_{\mathrm{up}}^{(g)}
\in
\mathbb{R}^{B \times K_g \times d^{\mathrm{stu}}},
\label{eq:cg_sparse}
\end{equation}
where
$W_{\mathrm{down}}^{(g)} \in \mathbb{R}^{d^{\mathrm{stu}} \times r_g}$,
$W_{\mathrm{up}}^{(g)} \in \mathbb{R}^{r_g \times d^{\mathrm{stu}}}$,
and $r_g \ll d^{\mathrm{stu}}$ is the bottleneck dimension.
Thus, the generalized branch is itself parameter-efficient and remains active for every input, capturing task-agnostic transferable structure.

\paragraph{Routed specialized adapters.}
To model input-dependent transfer patterns, we introduce $M$ specialized adapters
$\{\mathcal{E}_m\}_{m=1}^{M}$.
Each adapter is parameterized as a lightweight residual bottleneck operating on the specialized compact context:
\begin{equation}
\mathcal{E}_m(C_s^{0})
= C_s^{0}
+
\phi\big(
C_s^{0} W^{(m)}_{\mathrm{down}} 
\big) W^{(m)}_{\mathrm{up}} ,
\qquad
m=1,\dots,M,
\label{eq:adapter}
\end{equation}
where
$W^{(m)}_{\mathrm{down}} \in \mathbb{R}^{d^{\mathrm{stu}} \times r_s}$,
$W^{(m)}_{\mathrm{up}} \in \mathbb{R}^{r_s \times d^{\mathrm{stu}}}$,
and $r_s \ll d^{\mathrm{stu}}$ is the bottleneck dimension.
Thus, each specialized branch is also parameter-efficient, while the residual form preserves the base specialized context and adds only a low-rank input-adaptive correction.

\paragraph{Route-first sparse execution.}
Unlike dense multi-adapter aggregation, we first compute routing scores and execute only the selected adapters. Specifically, we summarize the teacher hidden states by average pooling, $h_r=\frac{1}{N}\sum_{i=1}^{N} H_T[:,i,:] \in \mathbb{R}^{B \times d^{\mathrm{tea}}}$, and feed it into a lightweight router to obtain routing probabilities $p=\mathrm{softmax}(\phi(h_r W_1)W_2) \in \mathbb{R}^{B \times M}$, where $H_T\in \mathbb{R}^{B \times N \times d^{\mathrm{tea}}}$, $W_1 \in \mathbb{R}^{d^{\mathrm{tea}} \times d^{\mathrm{gate}}}$, and $W_2 \in \mathbb{R}^{d^{\mathrm{gate}} \times M}$.
For each instance $b$, we select the top-$k$ adapters $\mathcal{I}_b=\mathrm{TopK}(p_b,k)$ and renormalize their routing weights as $\bar{p}_{b,m}=p_{b,m}/\sum_{j\in\mathcal{I}_b}p_{b,j}$ for $m\in\mathcal{I}_b$. The final specialized context is then computed by weighted sparse aggregation:
\begin{equation}
C_s[b,:,:]=\sum_{m\in\mathcal{I}_b}\bar{p}_{b,m}\,\mathcal{E}_m(C_s^{0}[b,:,:]) \in \mathbb{R}^{K_s \times d^{\mathrm{stu}}}.
\label{eq:cs_sparse}
\end{equation}
Stacking all instances yields $C_s \in \mathbb{R}^{B \times K_s \times d^{\mathrm{stu}}}$.

\paragraph{Transferred context and efficiency.}
We form the final transferred context by concatenating the generalized and specialized branches:
\begin{equation}
C_A
=
[\, C_g ; C_s \,]
\in
\mathbb{R}^{B \times (K_g+K_s) \times d^{\mathrm{stu}}}.
\label{eq:CA_sparse}
\end{equation}

The parameter efficiency comes from two sources.
First, both WAM backbones are frozen, so only the CKT module is trainable.
Second, the specialized branch is sparse: although we maintain $M$ adapters in total, only $k \ll M$ adapters are active per instance.
Let $P_{\mathrm{sh}}$, $P_g$, $P_e$, and $P_r$ denote the parameter counts of the shared trunk, generalized adapter, one specialized adapter, and the router, respectively.
Then there is 
$P_{\mathrm{train}}
=
P_{\mathrm{sh}} + P_g + M P_e + P_r $ and 
$ 
P_{\mathrm{active}}
=
P_{\mathrm{sh}} + P_g + k P_e + P_r,$
where $P_{\mathrm{active}}$ denotes the parameters involved in one forward pass.

\subsection{Student-Side Textual Context Injection}
\label{sec:student_injection}

Let $E_t \in \mathbb{R}^{B \times L_t \times d^{\mathrm{stu}}}$ denote the
student textual embeddings. We inject the transferred teacher context by
appending the aggregated context $C_A$ to the conditioning sequence:
\begin{equation}
\tilde{E}_t = [\,E_t \,;\, \langle\textsc{sep}\rangle \,;\, C_A\,]
\in \mathbb{R}^{B \times (L_t + K_g + K_s + 1) \times d^{\mathrm{stu}}},
\label{eq:inject_sparse}
\end{equation}
where $\langle\textsc{sep}\rangle$ is a learnable separator token. The
augmented sequence is fed into the frozen student backbone only through multi-head 
cross-attention:
\begin{equation}
Z_\tau^{(\ell+1)} = Z_\tau^{(\ell)} +
\mathrm{MHCA}\!\bigl(
Q=Z_\tau^{(\ell)},\,
K=\tilde{E}_t W_K^{(\ell)},\,
V=\tilde{E}_t W_V^{(\ell)}
\bigr).
\end{equation}
In this way, each future-frame patch receives a content-conditioned residual
from $C_A$ through the same conditioning pathway used during pretraining,
while leaving the denoising backbone unchanged.

Importantly, this design is also geometry-preserving. The student's 3D RoPE
is applied only in visual self-attention and depends solely on the
spatio-temporal indices of visual tokens, while cross-attention uses no
positional encoding on the conditioning side. Therefore, extending the
conditioning sequence from $E_t$ to $\tilde{E}_t$ does not alter the rotary
geometry of the visual stream. More details are provided in
Appendix~\ref{app:student_injection}.

\subsection{Training Objective}
During training, both the teacher WAM and the student WAM are frozen, and only the parameter-efficient CKT module is optimized. This preserves the pretrained dynamics of both models and constrains knowledge transfer to lightweight generalized adapters, a router, and sparsely activated specialized adapters.
We train CKT module with the student's original diffusion-style latent objective over both future-video and future-action latents. Let $v_0 \in \mathbb{R}^{B \times T_v \times d_v}$ be the clean VAE-compressed future-video latent sequence, and let $a_0 \in \mathbb{R}^{B \times T_a \times d_a}$ be the clean future-action sequence. For each training instance $b$, we sample a shared noise level from the base model's log-normal distribution,
$
\ln(\sigma_b)\sim\mathcal{N}(P_{\mathrm{mean}},P_{\mathrm{std}}^2),
$
then draw modality-specific Gaussian noise
$
\epsilon_b^{v},\epsilon_b^{a}\sim\mathcal{N}(0,I).
$
The noisy latents are constructed as
\begin{equation}
v_{n,b}=\mu(v_{0,b},\sigma_b)+\sigma_b\odot\epsilon_b^{v},\qquad
a_{n,b}=\mu(a_{0,b},\sigma_b)+\sigma_b\odot\epsilon_b^{a},
\label{eq:noisy_va}
\end{equation}
where $\mu(\cdot,\cdot)$ is determined by the underlying SDE parameterization. Using the same $\sigma_b$ aligns the corruption level across video and action, while preserving modality-specific perturbations.
Conditioned on $\tilde{E}_t$, the frozen student predicts the clean future action $\hat{a}_0=D_a(a_{n,b};\tilde{E}_t)$ and clean future video $\hat{v}_0=D_v(v_{n,b};\tilde{E}_t)$. The corresponding losses are
\begin{align}
\mathcal{L}_{\mathrm{act}}
=
\frac{1}{B}\sum_{b=1}^{B}\sum_{t=1}^{T_a}
m^{a}_{b,t}\,\omega(\sigma_b)\|\hat{a}_{0,b,t}-a_{0,b,t}\|_2^2, \ \
\mathcal{L}_{\mathrm{vid}}
=
\frac{1}{B}\sum_{b=1}^{B}\sum_{t=1}^{T_v}
m^{v}_{b,t}\,\omega(\sigma_b)\|\hat{v}_{0,b,t}-v_{0,b,t}\|_2^2,
\label{eq:loss_vid}
\end{align}
where $\omega(\sigma_b)$ is the noise-dependent weight, $m^{a}_{b,t}$ and $m^{v}_{b,t}$ are binary masks for valid action and future-video positions, and $T_a$ and $T_v$ denote the action-chunk and future-video lengths.

To avoid routing collapse, we further use an MoE-style load-balancing loss. Let
$
P_m=\frac{1}{B}\sum_{b=1}^{B}p_{b,m}
$
be the batch-averaged routing mass of adapter $m$, and
$
f_m=\frac{1}{kB}\sum_{b=1}^{B}\mathbb{I}\{m\in\mathcal{I}_b\}
$
be its empirical top-$k$ selection frequency, where
$
\mathcal{I}_b=\operatorname{Top\text{-}K}(\{p_{b,m}\}_{m=1}^{M},k).
$
The auxiliary load balancing loss is
$
\mathcal{L}_{\mathrm{bal}}=M\sum_{m=1}^{M}f_mP_m,
$
and the overall objective is
\begin{equation}
\mathcal{L}=\mathcal{L}_{\mathrm{act}}+\lambda_{\mathrm{vid}}\mathcal{L}_{\mathrm{vid}}+\lambda_{\mathrm{bal}}\mathcal{L}_{\mathrm{bal}},
\end{equation}
where $\lambda_{\mathrm{vid}}$ and $\lambda_{\mathrm{bal}}$ control the weights of future-video supervision and load balancing.


\section{Simulation Experiments}

\subsection{Experimental Settings}
\label{exp_settings}

\paragraph{Model setup.}
We use {Cosmos-Policy-2B} as the student WAM and {DreamZero-14B} as the teacher WAM.
The teacher and the student are kept frozen throughout training, while only the proposed CKT modules are optimized.
For knowledge transfer, we use the hidden states from the {$\ell^*=20$}-th layer of the teacher, which has 40 layers in total. This choice provides an intermediate representation that is sufficiently informative for transfer, while avoiding the stronger task-specific bias often present in higher layers.

\paragraph{Training setup.}
All models are trained on 8 NVIDIA A800 GPUs with distributed training in bfloat16.
We use AdamW with $\beta_1=0.9$, $\beta_2=0.95$, weight decay $0.01$, and a cosine learning-rate schedule with 1k warmup steps.
Unless otherwise stated, the learning rate is set to \textbf{$3\times10^{-4}$} for adapter modules.
For the shared projection in Eq.~\eqref{eq:shared_proj}, we use bottleneck dimension $d_b=512$.
For the compression of the teacher WAM's knowledge, we use two learnable query banks of size $K_g=K_s=32$, with 16 attention heads and 0.1 dropout rate.
The compressed teacher context is computed once and reused by all downstream branches.
Unless otherwise stated, we use \textbf{$M=8$} routed branches with top-$2$ selection.
The trainable parameters are only 187.4M (1.17\%) parameters in the CKT module.  
We set $\mu(v_{0,b},\sigma_b)=v_{0,b}$ and $\mu(a_{0,b},\sigma_b)=a_{0,b}$.
The noise level is set as $\ln(\sigma) \sim \mathcal{N}(P_{\mathrm{mean}}, P_{\mathrm{std}}^{2})$, where $P_{\mathrm{mean}} = 1.39$ and $P_{\mathrm{std}} = 1.2$.
We set the loss scaling factors as $\lambda_{\mathrm{vid}}=1$ and $\lambda_{\mathrm{bal}}=0.01$.

\subsection{LIBERO-Plus Results}

\begin{table*}[t]
  \centering
  \caption{LIBERO-Plus benchmark zero-shot performance. Results are shown in success rate (\%).}
  \label{1}
  \resizebox{\textwidth}{!}{%
    \begin{tabular}{l|ccccccc|c}
      \toprule
      Model & Camera & Robot & Language & Light & Background & Noise & Layout & Total \\
      \midrule
      OpenVLA \citep{kim2025openvla}   & 0.8  & 3.5  & 23.0 & 8.1  & 34.8 & 15.2 & 28.5 & 15.6 \\
      WorldVLA \citep{cen2025worldvla}  & 0.1  & 27.9 & 41.6 & 43.7 & 17.1 & 10.9 & 38.0 & 25.0 \\
      NORA \citep{hung2025nora}      & 2.2  & 37.0 & 65.1 & 45.7 & 58.6 & 12.8 & 62.1 & 39.0 \\
      UniVLA \citep{bu2025univla}    & 1.8  & 46.2 & 69.6 & 69.0 & 81.0 & 21.2 & 31.9 & 42.9 \\
      $\pi_0$ \citep{black2024pi0}   & 13.8 & 6.0  & 58.8 & 85.0 & 81.4 & 79.0 & 68.9 & 53.6 \\
      $\pi_0$-Fast \citep{pertsch2025fast} & 65.1 & 21.6 & 61.0 & 73.2 & 73.2 & 74.4 & 68.8 & 61.6 \\
      $\pi_{0.5}$ \citep{physicalintelligence2025pi05} & 75.4 & 77.5 & 85.6 & 96.9 & 94.6 & 89.7 & 85.7 & 85.7 \\
      OpenVLA-OFT \citep{kim2025finetuning} & 56.4 & 31.9 & 79.5 & 88.7 & 93.3 & 75.8 & 74.2 & 69.6 \\
      RIPT-VLA \citep{tan2025interactive} & 55.2 & 31.2 & 77.6 & 88.4 & 91.6 & 73.5 & 74.2 & 68.4 \\
      X-VLA \citep{zheng2025xvla} & 23.4 & \textbf{89.7} & 75.7 & 88.2 & \textbf{96.0} & 62.7 & 71.8 & 71.4 \\
      HoloBrain0-GD \citep{lin2026holobrain0} & 65.5 & 58.2 & 78.7 & 88.1 & 90.3 & 66.9 & 79.5 & 74.0 \\
      VLA-JEPA \citep{sun2026vlajepa} & 64.2 & 67.7 & 88.1 & 91.8 & 93.4 & 65.8 & 83.9 & 77.9 \\
      GE-Act \citep{liao2025genieenvisioner} & 60.7 & 77.0 & 77.4 & 95.8 & 86.0 & 90.9 & 80.2 & 80.3 \\
      VLANeXt \citep{wu2026vlanext}  & 71.6 & 63.4 & 81.4 & 93.7 & 91.8 & 89.5 & 77.7 & 80.1 \\
      ABot-M0 \citep{yang2026abotm0} & 60.4 & 67.9 & 86.4 & 96.2 & 91.6 & 86.4 & 82.6 & 80.5 \\
      Cosmos-Policy \citep{kim2026cosmospolicy} & {75.8} & 63.3 & 81.7 & 96.5 & 88.9 & 92.7 & 82.2 & 82.2 \\
      \textbf{CKT-WAM (Ours)} &  \textbf{77.4} & {71.4} & \textbf{86.7} & \textbf{98.2} & {90.2} & \textbf{94.8} & \textbf{88.5} & \textbf{86.1} \\
      \bottomrule
    \end{tabular}
  }
\end{table*}

We use LIBERO-Plus \citep{libero_plus} as the simulation benchmark.
The student WAM is initialized from the official Cosmos-Policy-LIBERO-Predict2-2B checkpoint.
All models are trained on the combined training sets from all four suites of LIBERO \citep{libero}, and evaluated in the zero-shot setting under seven out-of-distribution dimensions: camera, robot, language, lighting, background, noise, and layout.

As shown in Table~\ref{1}, CKT-WAM achieves the best overall zero-shot performance on LIBERO-Plus. Its advantage is not limited to a single perturbation type: the model attains the best results on camera, language, lighting, layout, and total success rate, while remaining competitive on the other dimensions. In particular, compared with prior VLA and WAM baselines, CKT-WAM shows stronger robustness under semantic variation and scene-level distribution shifts, suggesting that the transferred teacher knowledge improves both instruction grounding and generalization to unseen visual conditions. Although some competing methods are stronger on individual dimensions such as camera, robot, background, or noise, CKT-WAM delivers the most balanced performance across all out-of-domain (OOD) factors, which ultimately leads to the strongest aggregate generalization.

\subsection{Comparison with Parameter Efficient Adaptation Methods}

To isolate the effect of different parameter-efficient adaptation strategies, we compare our method with a set of representative PEFT baselines, including LoRA \citep{hu2022lora}, Mona \citep{yin2025mona}, MiLoRA \citep{wang-etal-2025-milora}, PiSSA \citep{meng2024pissa}, DoRA \citep{liu2024dora}, GOAT \citep{fan2025goat}, and AdaMoLE \citep{liu2024adamole}.
For all non-ours baselines, the hidden states extracted from the teacher WAM's $\ell^*$-th layer are first projected by a two-layer MLP and then concatenated with the student WAM's textual input along the token dimension.
When training baselines, the first $\ell^*$-th layers of teacher WAM 
and the whole student WAM are adapted using the corresponding PEFT method.
This setting keeps the teacher-student interaction interface fixed across baselines, so that the comparison mainly reflects the effectiveness of the adaptation mechanism itself.

As shown in Table~\ref{tab:peft_compare}, CKT-WAM achieves the best overall performance among all PEFT baselines while using the smallest fraction of trainable parameters.
Specifically, it improves the total success rate from 82.8\% to 86.1\% compared with LoRA, while reducing the trainable parameter ratio from 2.59\% to 1.17\%.
It also consistently outperforms stronger recent PEFT methods such as AdaMoLE and GOAT, indicating that the gain does not come from increased adaptation capacity, but from a more effective mechanism for transferring teacher knowledge into the student WAM.
The advantage of CKT-WAM is particularly clear under \emph{robot}, \emph{light}, and \emph{layout} shifts, where it achieves the best performance, and it remains highly competitive on the other dimensions.
Compared with full fine-tuning, our method attains a nearly identical overall success rate (86.1\% vs.\ 86.3\%) using only 1.17\% trainable parameters, while even surpassing it on \emph{robot}, \emph{light}, and \emph{layout}.
Although full fine-tuning remains slightly stronger on \emph{language}, \emph{background}, and \emph{noise}, the gap in overall performance is marginal, suggesting that CKT-WAM captures most of the benefit of full adaptation at a tiny fraction of the optimization cost.
Overall, these results demonstrate that CKT-WAM provides a substantially better trade-off between adaptation efficiency and zero-shot generalization than existing PEFT alternatives, while approaching the performance of full fine-tuning.

\begin{table*}[t]
  \centering
  \caption{Comparison with parameter-efficient adaptation methods on LIBERO-Plus.
  The column \# Parameters denotes the ratio of trainable parameters to the overall parameters.}
  \label{tab:peft_compare}
  \resizebox{\textwidth}{!}{%
    \begin{tabular}{l|c|ccccccc|c}
      \toprule
      Model & \# Parameters & Camera & Robot & Language & Light & Background & Noise & Layout & Total \\
      \midrule
      Full Fine-tuning & {100\%} & {77.6} & {71.2} & \textbf{87.5} & {97.1} & \textbf{91.0} & \textbf{95.7} & {88.3} & \textbf{86.3} \\
      \midrule 
LoRA \citep{hu2022lora}                  & 2.59\% & 76.1 & 67.9 & 81.6 & 95.9 & 88.1 & 90.4 & 84.7 & 82.8 \\
PiSSA \citep{meng2024pissa}             & 2.04\% & 75.2 & 68.5 & 84.2 & 96.1 & 88.4 & 90.6 & 85.1 & 83.3 \\
DoRA \citep{liu2024dora}                & 2.06\% & 76.4 & 67.8 & 84.7 & 96.4 & 88.7 & 90.8 & 85.4 & 83.6 \\
MiLoRA \citep{wang-etal-2025-milora}    & 1.73\% & 73.9 & 71.2 & 85.3 & 96.8 & 89.0 & 91.0 & 85.9 & 84.0 \\
Mona \citep{yin2025mona}            & 1.65\% & \textbf{79.2} & 70.6 & 84.1 & 97.0 & 89.3 & 91.1 & 86.2 & 84.7 \\
AdaMoLE \citep{liu2024adamole}          & 1.69\% & 76.0 & 70.9 & 85.4 & 97.3 & 89.6 & 91.3 & 86.7 & 84.6 \\
GOAT \citep{fan2025goat}                & 1.77\% & 75.0 & 71.0 & 86.1 & 97.7 & 89.9 & 91.6 & 87.4 & 84.8 \\
\textbf{CKT-WAM (Ours)} & \textbf{1.17\%} & {77.4} & \textbf{71.4} & {86.7} & \textbf{98.2} & {90.2} & {94.8} & \textbf{88.5} & {86.1} \\
      \bottomrule
    \end{tabular}
  }
\end{table*}

\subsection{Ablation Study}

To understand the contribution of each component in CKT-WAM, we perform ablations on LIBERO-Plus under the same zero-shot evaluation protocol. Specifically, 
``w.o. Generalized Adapter'' removes the generalized adapter branch and keeps only the specialized adapter, 
``w.o. Specialized Adapters'' removes the specialized adapter branch and retains only the generalized adapters, 
and ``w.o. Auxiliary Loss'' drops the auxiliary regularization while keeping the full architecture unchanged.

\begin{table*}[t]
  \centering
  \caption{Ablation study on LIBERO-Plus benchmark zero-shot performance. Results are shown in success rate (\%).}
  \label{tab:ablation_libero_plus}
  \resizebox{\textwidth}{!}{%
    \begin{tabular}{l|ccccccc|c}
      \toprule
      Model & Camera & Robot & Language & Light & Background & Noise & Layout & Total \\
      \midrule
      w.o. Context & {75.8} & 63.3 & 81.7 & 96.5 & 88.9 & 92.7 & 82.2 & 82.2 \\      
w.o. Generalized Adapter    & 73.5 & 70.8 & 82.4 & 96.7 & 86.1 & 89.4 & 86.9 & 83.0 \\
w.o. Specialized Adapters   & 72.8 & 68.1 & 82.3 & 97.0 & 88.7 & 90.6 & 85.7 & 82.7 \\
w.o. Auxiliary Loss         & 75.0 & {70.9} & 83.6 & 97.6 & 86.9 & 92.2 & 86.9 & 84.1 \\
\midrule
\textbf{CKT-WAM (Ours)}     & \textbf{77.4} & \textbf{71.4} & \textbf{86.7} & \textbf{98.2} & \textbf{90.2} & \textbf{94.8} & \textbf{88.5} & \textbf{86.1} \\
      \bottomrule
    \end{tabular}
  }
\end{table*}

As shown in Table~\ref{tab:ablation_libero_plus}, all three components of CKT-WAM contribute to the final zero-shot generalization performance, while their roles are not identical. Removing the specialized adapters causes the largest overall degradation, which suggests that instance-dependent specialization is particularly important for handling the diverse distribution shifts in LIBERO-Plus. Removing the generalized adapter also leads to a clear drop, indicating that the shared transfer branch provides a strong and stable foundation for cross-task knowledge injection. In contrast, removing the auxiliary loss yields the smallest degradation. This suggests that the auxiliary loss mainly acts as a regularizer that improves routing stability and adapter coordination, while the principal gains come from the joint use of generalized and specialized transfer pathways. Overall, these results validate that the full CKT-WAM design is necessary to obtain the strongest aggregate performance and the most balanced robustness across diverse OOD shifts.


We ablate the teacher intermediate layer $\ell^*$ for context extraction under the LIBERO-Plus zero-shot setting. Figure~\ref{fig:layer_ablation_latency} shows that success rate improves from $82.1\%$ at $\ell^*=0$ to $86.1\%$ at $\ell^*=20$, then saturates, while latency grows nearly linearly from $0.114$ s to $0.587$ s per action chunk. This suggests that intermediate teacher features provide the best balance between semantic usefulness and computational cost. We thus adopt $\ell^*=20$ as the default setting.

\begin{figure}[t]
\centering
\includegraphics[width=0.9\linewidth]{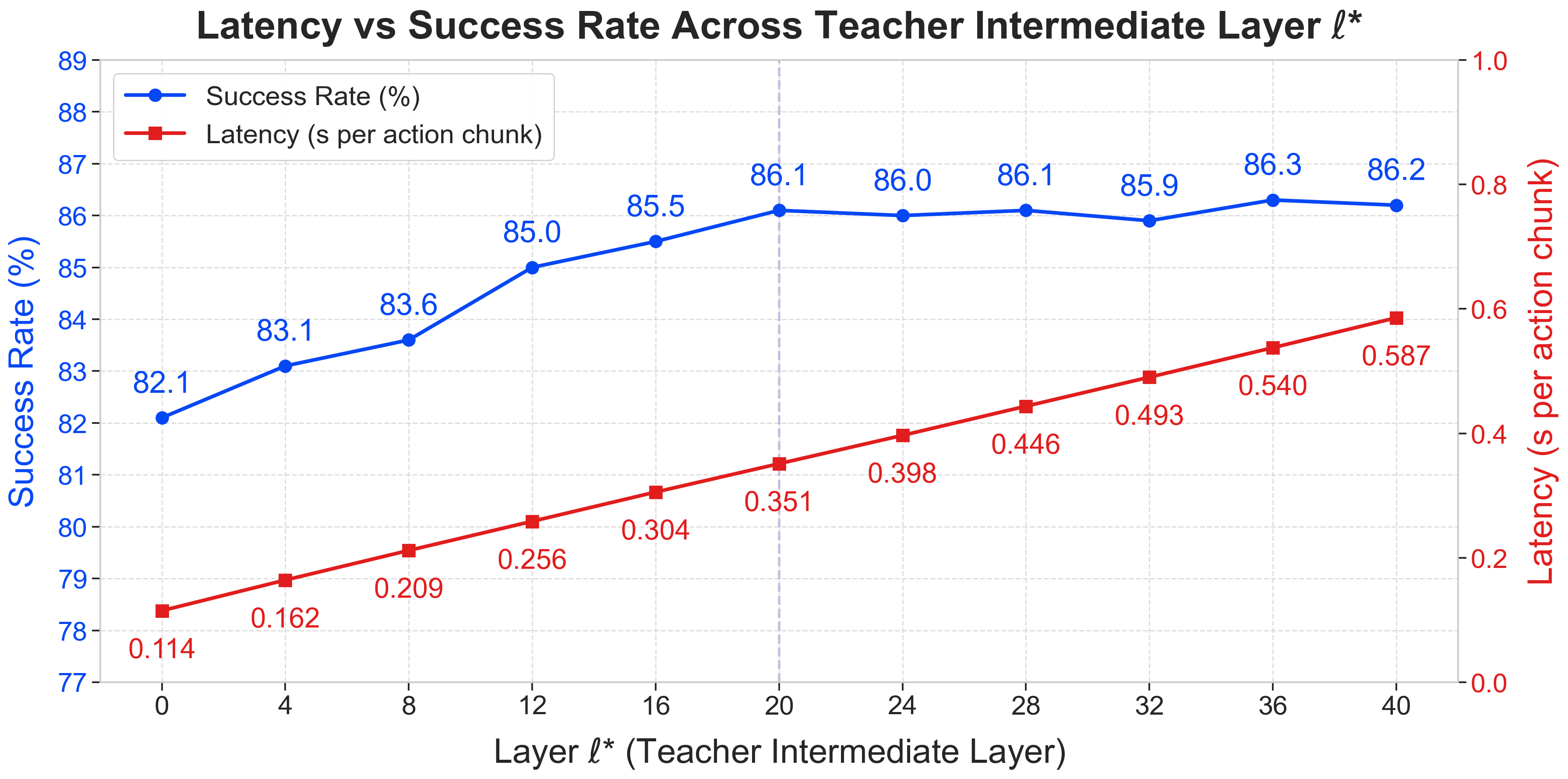}
\caption{
Latency-success trade-off across different teacher intermediate layers $\ell^*$.
The left y-axis reports the zero-shot success rate on LIBERO-Plus, while the right y-axis reports inference latency in seconds per action chunk.
}
\label{fig:layer_ablation_latency}
\end{figure}

\section{Real-World Experiments}
\label{subsec:real_robot_setting}

\begin{figure*}[t]
    \centering
    \includegraphics[width=\textwidth]{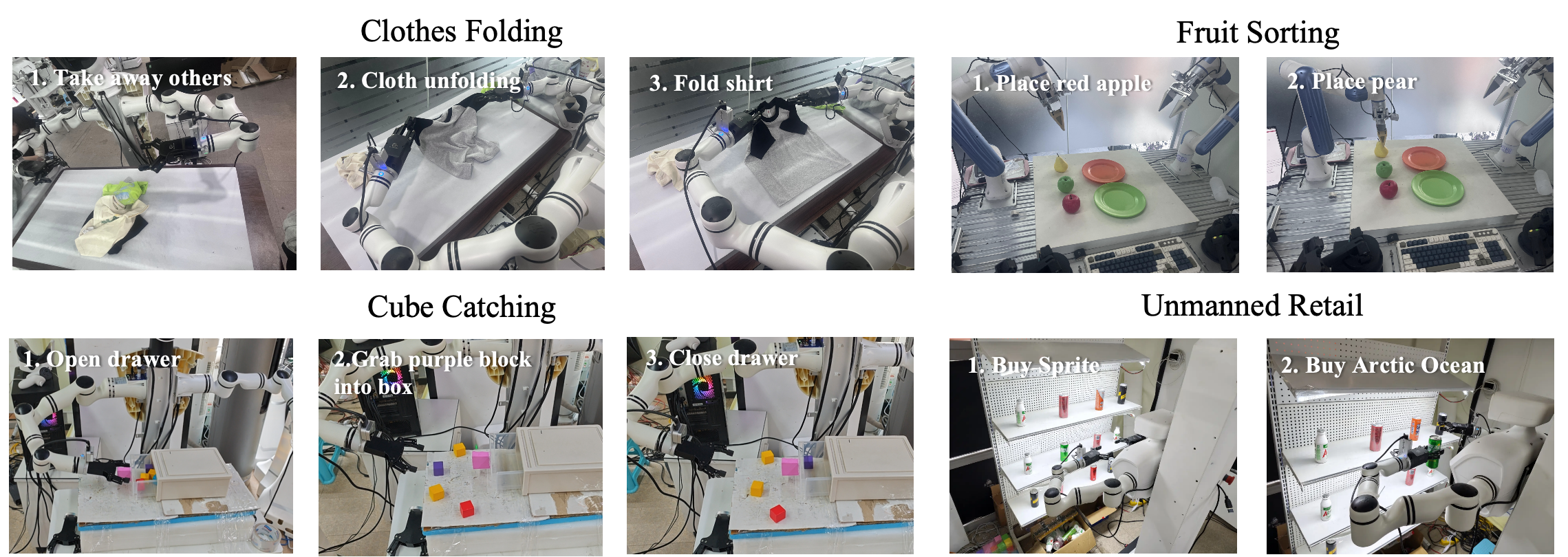}
    \caption{
    Real-world long-horizon tasks in our evaluation.
    From left to right: (\textbf{1}) clothes folding, including removing distracting clothes, unfolding the garment, and folding it;
    (\textbf{2}) fruit sorting, placing the red apple and pear onto the designated plates;
    (\textbf{3}) cube catching, including opening the drawer, placing the purple block into the box, and closing the drawer;
    and (\textbf{4}) unmanned retail, picking target drinks from a cluttered shelf.
    All tasks require completing multiple ordered subgoals.
    }
    \label{fig:real_robot}
\end{figure*}

Our real-world evaluation focuses on \emph{long-horizon, multi-step manipulation} rather than single-step skills. As shown in Figure~\ref{fig:real_robot}, we evaluate four representative tasks: clothes folding, fruit sorting, cube catching, and unmanned retail. All tasks require the policy to complete a sequence of dependent subgoals under continuous visual feedback. For each task, we conduct {45 trials} in the real world and report the task success rate. Among these tasks, \emph{clothes folding} is particularly challenging because it involves \textbf{three consecutive stages}: removing distracting clothes, unfolding the target garment, and folding it. Failure at any intermediate stage can prevent successful task completion, making this setting a practical test of long-horizon real-world control. Detailed task descriptions and manipulation process are shown in Appendix~\ref{real_robot_setting_detailed}.

\begin{table}[t]
\centering
\caption{Real-world success rates (\%) on four long-horizon and multi-step manipulation tasks.}
\label{tab:real_world}
\resizebox{\linewidth}{!}{
\begin{tabular}{l|cccc|c}
\toprule
Method 
& Clothes Folding 
& Fruit Sorting 
& Cube Catching 
& Unmanned Retail 
& Avg. \\
\midrule
$\pi_{0.5}$ \citep{physicalintelligence2025pi05} 
& \textbf{77.8} 
& 84.4 
& 80.0 
& 77.8 
& 80.0 \\
Cosmos Policy \citep{kim2026cosmospolicy} 
& 62.2 
& 77.8 
& 68.9 
& 71.1 
& 70.0 \\
Fast-WAM \citep{yuan2026fastwam} 
& 66.7 
& 75.6 
& 73.3 
& 66.7
& 70.6 \\
\textbf{CKT-WAM (Ours)} 
& 73.3 
& \textbf{88.9} 
& \textbf{86.7} 
& \textbf{84.4} 
& \textbf{83.3} \\
\bottomrule
\end{tabular}
}
\end{table}

As shown in Table~\ref{tab:real_world}, {CKT-WAM} achieves the best overall real-world performance, obtaining the highest average success rate of {83.3\%} across the four long-horizon tasks. Among the baselines, $\pi_{0.5}$ is the strongest competitor with an average success rate of 80.0\%, and it slightly surpasses our method on the most challenging \emph{clothes folding} task. Nevertheless, {CKT-WAM} achieves the best success rate on the other three tasks, namely \emph{fruit sorting}, \emph{cube catching}, and \emph{unmanned retail}. These results suggest that CKT-WAM provides stronger long-horizon planning and multi-stage execution ability in real-world settings, while remaining competitive on highly challenging deformable-object manipulation tasks.

\section{Conclusion}
We presented {CKT-WAM}, a parameter efficient framework for transferring knowledge between world action models through a compact contextual interface. By combining intermediate-layer teacher representations, learnable-query compression, generalized and specialized sparse adapters, and lightweight conditioning-side injection, our method enables effective teacher--student transfer while keeping both backbones frozen. 
Experiments on LIBERO-Plus show that CKT-WAM delivers strong zero-shot generalization, outperforms representative PEFT baselines, and approaches full fine-tuning performance while updating only 1.17\% of the parameters. Beyond simulation, CKT-WAM also demonstrates strong real-world long-horizon manipulation ability, achieving the best average success rate of 83.3\% across four multi-step tasks. Together, these results highlight compact contextual transfer as a simple, efficient, and effective approach to enhancing WAMs.

\bibliographystyle{plainnat} 
\small
\bibliography{main}
\normalsize

\newpage
\appendix

\section{Technical Appendices and Supplementary Material}

\subsection{Statement for Use of LLMs}

LLMs were only used to assist with language polishing in certain sections of this paper.

\subsection{Reproducibility Statement}
We provide open-source code to reproduce all experiments.
To facilitate reproducibility, we additionally include a minimal working example and environment setup instructions, enabling researchers to verify correctness and replicate our reported results with minimal effort.
Code is available at \url{https://github.com/YuhuaJiang2002/CKT-WAM}.

\subsection{Impact Statement}
This paper presents work whose goal is to advance the field of machine learning. We do not identify any specific impacts of this work that require particular emphasis here.

\subsection{Limitations and Future Work}
A limitation of this work is that our study is currently focused on embodied control and WAM-to-WAM transfer, and we do not explore whether the proposed context knowledge transfer mechanism can generalize to broader multimodal settings such as text modeling, visual question answering, or other non-robotic generative tasks. We view this as a scope choice rather than a fundamental restriction of the method, since the core design of CKT-WAM---compressing teacher-side intermediate representations into a compact transferable context and injecting it through a lightweight conditioning interface---is not inherently tied to robotics. An important direction for future work is therefore to investigate whether the same principle can serve as a general parameter-efficient transfer mechanism across heterogeneous foundation models in language-only and vision-language domains, especially in settings where full hidden-state alignment is costly or structurally mismatched.

\section{More Experimental Results}
\subsection{Stage-wise Selection Patterns of Specialized Adapters}
\label{app:stagewise_selection_ckt}

To further examine whether the routed specialized adapters learn meaningful functional roles, we analyze their selection patterns within a single long-horizon manipulation task, \emph{cube catching}. This task is naturally decomposed into four sequential stages: \emph{drawer opening}, \emph{cube grasping}, \emph{cube placing}, and \emph{drawer closing}. We randomly sample 120 episodes from this task and, for each stage, compute the average selection probability of each specialized adapter.

As shown in Figure~\ref{fig:selection_ckt}, the specialized adapters exhibit clear stage-dependent routing behavior even within the same task. Different stages concentrate probability mass on different small subsets of adapters, rather than relying on a fixed task-level activation pattern. Specifically, \emph{drawer opening} mainly activates Adapters 2 and 3, \emph{cube grasping} is dominated by Adapters 5 and 7, \emph{cube placing} concentrates on Adapters 5 and 1, and \emph{drawer closing} shifts the focus primarily to Adapters 6 and 4. Moreover, each stage mainly activates only one or two adapters, while the remaining adapters receive much lower probabilities.

These results suggest that the routed specialized adapters do not behave as static modules assigned to the entire task. Instead, their activation changes dynamically with the current manipulation sub-stage, indicating fine-grained functional specialization within a single trajectory. The consistently sparse probability concentration also supports the effectiveness of our route-first sparse design, where only a small subset of experts is needed to capture the control demands of each stage.

\begin{figure}[t]
    \centering
    \includegraphics[width=\linewidth]{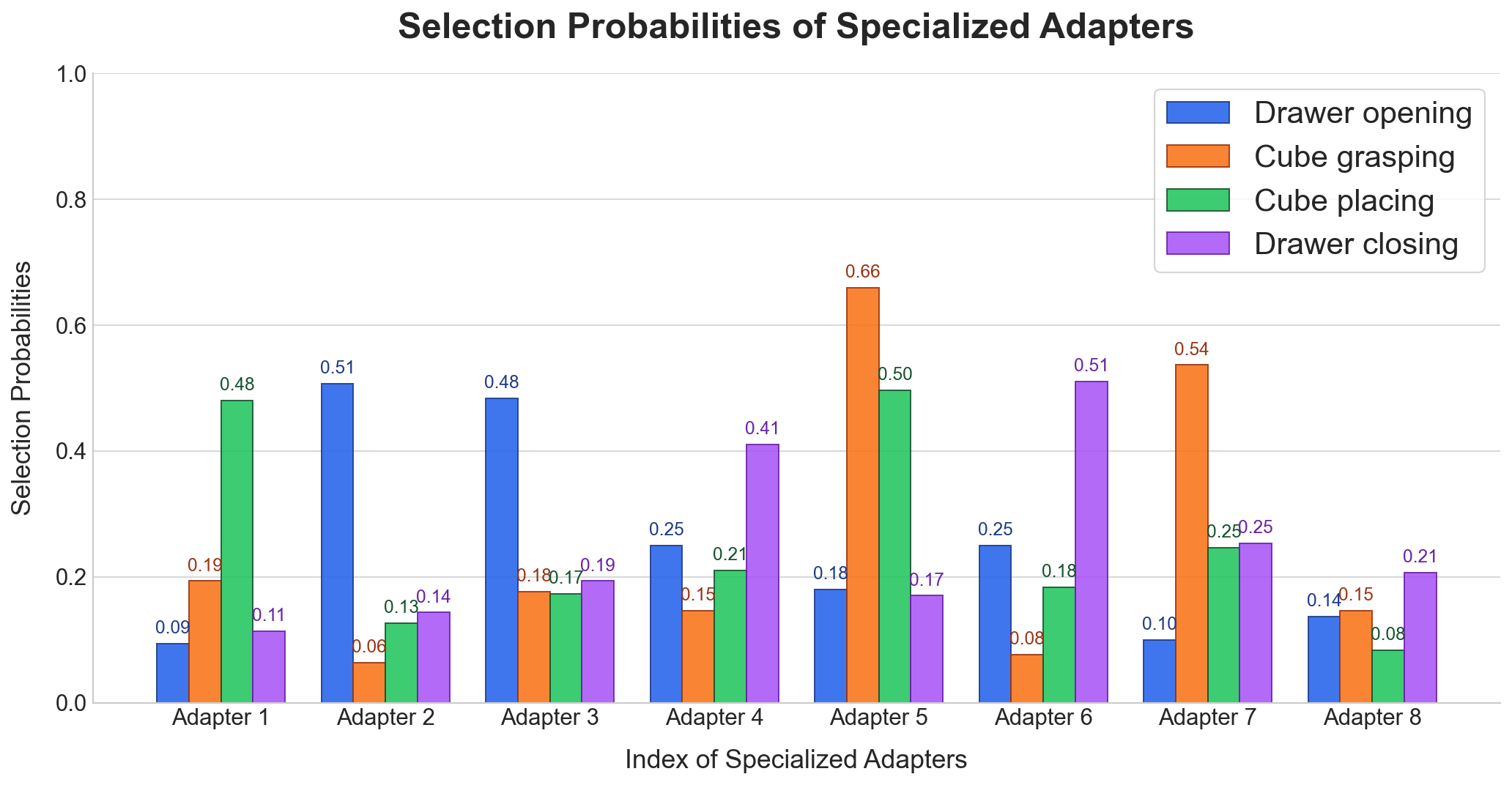}
    \caption{
    Stage-wise selection probabilities of specialized adapters in the \emph{cube catching} task.
    We sample 120 episodes and compute the average adapter selection probabilities for the four stages:
    \emph{drawer opening}, \emph{cube grasping}, \emph{cube placing}, and \emph{drawer closing}.
    Even within the same task, the dominant adapters vary substantially across stages, and each stage mainly activates only one or two experts.
    }
    \label{fig:selection_ckt}
\end{figure}

\subsection{Experimental Settings}
\label{app:exp_settings}

Table~\ref{tab:appendix_exp_settings} summarizes the default experimental settings used throughout this work unless otherwise specified.
Our setup follows a standard large-teacher/smaller-student knowledge transfer paradigm, where a frozen {DreamZero-14B} teacher supervises a {Cosmos-Policy-2B} student.
For transfer, we extract the hidden states from the intermediate teacher layer $\ell^*=20$ out of the 40-layer teacher, which provides a favorable trade-off between semantic richness and transferability.
The compressed teacher context is computed once and reused by all downstream branches, which is also consistent with the efficiency-oriented design of our method.

\begin{table}[t]
\centering
\caption{Default experimental settings used in our method unless otherwise specified.}
\label{tab:appendix_exp_settings}
\resizebox{\linewidth}{!}{%
\begin{tabular}{l|l}
\toprule
\textbf{Category} & \textbf{Setting} \\
\midrule
Student WAM & Cosmos-Policy-2B \\
Teacher WAM & DreamZero-14B \\
Teacher depth & 40 layers \\
Teacher hidden states selection from layer &  $\ell^* = 20$ \\
Precision & \texttt{bfloat16} \\
Hardware & 8 $\times$ NVIDIA A800 GPUs \\
Training mode & Distributed training \\
Optimizer & AdamW \\
AdamW $\beta_1, \beta_2$ & $\beta_1=0.9,\ \beta_2=0.95$ \\
Weight decay & $0.01$ \\
Learning-rate schedule & Linear Warmup and Cosine decay \\
Warmup steps & 1k \\
Learning rate & $3 \times 10^{-4}$ for adapter modules \\
Bottleneck dimension $d_b$ & 512 \\
Query bank size $K_g,K_s$ & $K_g=K_s=32$ \\
Attention heads & 16 \\
Dropout & 0.1 \\
Teacher hidden states extraction & Computed once and reused for all student's timesteps \\
Number of specialized adapters $M$ & 8 \\
Routing strategy & top-$k = 2$ \\
Noise level $\ln(\sigma) \sim \mathcal{N}(P_{\mathrm{mean}}, P_{\mathrm{std}}^{2})$ & $P_{\mathrm{mean}} = 1.39$ and $P_{\mathrm{std}} = 1.2$ \\
Video loss scaling factor $\lambda_{\mathrm{vid}}$ & $\lambda_{\mathrm{vid}}=1$ \\
Load balancing loss scaling factor $\lambda_{\mathrm{bal}}$ & $\lambda_{\mathrm{bal}}=0.01$ \\
\bottomrule
\end{tabular}%
}
\end{table}

Unless otherwise stated, all experiments use the configuration in Table~\ref{tab:appendix_exp_settings}.
In particular, the routing module uses $M=8$ branches with top-$k=2$ selection, while the CKT module adopts a bottleneck projection with $d_b=512$ and a shared query bank of size $K=32$.
These choices provide a good balance between transfer capacity, routing flexibility, and computational efficiency in our experiments.  

\subsection{Detailed Real-World Evaluation Setting}
\label{real_robot_setting_detailed}

Our real-world evaluation is designed to test \emph{long-horizon, multi-step manipulation} rather than isolated one-step behaviors. In all tasks, the robot must complete a sequence of ordered subgoals under continuous visual feedback, and an episode is counted as successful only when all required stages are completed in the correct order. This setting is more challenging than standard short-horizon evaluation because an error in an early stage can propagate and prevent successful completion of the later ones.

\begin{figure*}[t]
    \centering
    \includegraphics[width=\textwidth]{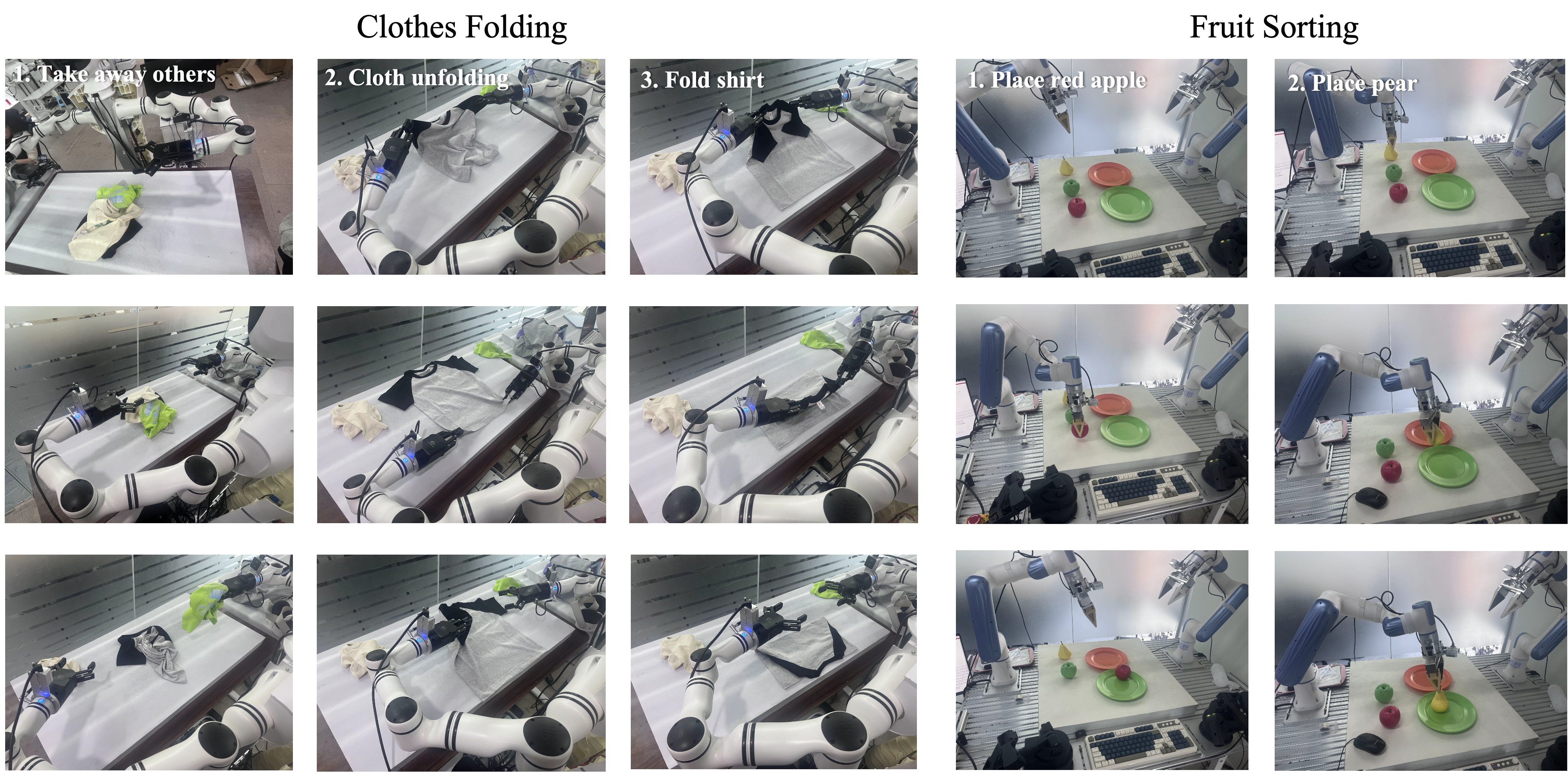}
    \caption{
    Detailed real-world examples for \textbf{Clothes Folding} and \textbf{Fruit Sorting}.
    The left three columns show representative intermediate states of the clothes-folding task, while the right two columns show the fruit-sorting task.
    }
    \label{fig:appendix_real_robot1}
\end{figure*}

As shown in Figure~\ref{fig:appendix_real_robot1}, the \textbf{Clothes Folding} task is a representative long-horizon manipulation problem with three explicitly ordered stages. First, the robot removes distracting clothes from the workspace to expose the target garment and create a clean manipulation area. Second, it unfolds the target shirt from a partially entangled or crumpled configuration into a flatter state that is suitable for precise bimanual interaction. Third, it performs the final folding motion to produce a compact folded shirt. This task is challenging because it requires sequential planning across multiple stages, large pose changes of deformable objects, and stable execution over an extended action horizon.

Figure~\ref{fig:appendix_real_robot1} also illustrates the \textbf{Fruit Sorting} task. In this task, the robot must identify the target fruits, grasp them one by one, and place them onto the designated plates in sequence. The task is not a single pick-and-place action: after placing the first fruit, the policy must preserve the intermediate arrangement, re-orient to the second target, and complete the remaining placement without disturbing previously placed objects. This requires sustained visual grounding and consistent progress across multiple dependent subgoals.

\begin{figure*}[t]
    \centering
    \includegraphics[width=\textwidth]{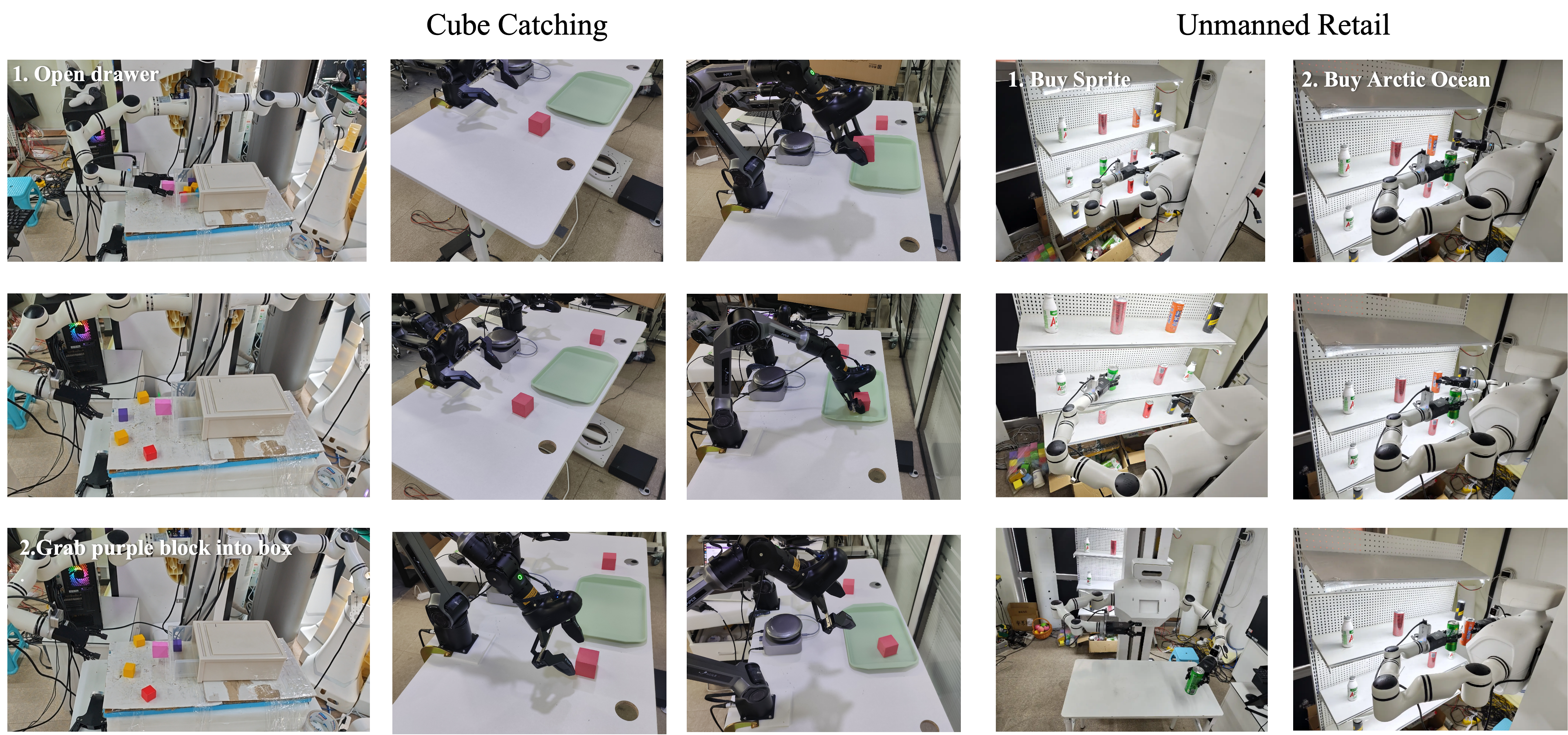}
    \caption{
    Detailed real-world examples for \textbf{Cube Catching} and \textbf{Unmanned Retail}.
    The left part shows representative snapshots of the cube task, and the right part shows the unmanned-retail task with target beverage selection from a cluttered shelf.
    }
    \label{fig:appendix_real_robot2}
\end{figure*}

Figure~\ref{fig:appendix_real_robot2} shows the \textbf{Cube Catching} task. In this task, the robot must first localize the target cube, move the gripper to a suitable pre-grasp pose, establish a stable grasp, and then transport the cube to the designated target area. Although the manipulated object is rigid and smaller than a garment, the task still requires several dependent stages, including accurate perception, grasp alignment, secure lifting, and controlled release. The evaluation therefore tests whether the policy can reliably chain together multiple primitive skills into a complete manipulation sequence.

The same figure further presents the \textbf{Unmanned Retail} task, where the robot interacts with a cluttered shelf containing multiple beverages. Given a target item request, the robot must visually identify the correct product, reach into the shelf without colliding with surrounding objects, grasp the target drink, and retrieve it from the shelf. In multi-item settings, this process is repeated sequentially for different requested beverages. This task is particularly challenging because it jointly requires category-level recognition, motion planning in clutter, stable grasping in constrained spaces, and successful completion of multiple ordered retrieval steps.

\subsection{Trainable Parameter Budget of the CKT Module}
\label{app:param-budget}

In this section we provide a detailed accounting of the parameters
introduced by our sparse CKT module.  All numbers correspond to
the default recipe used in the main experiments: a DreamZero--14B
teacher WAM with hidden dimension $d_{\text{tea}}{=}5120$, a
Cosmos-Policy--2B student WAM with hidden dimension
$d_{\text{stu}}{=}2048$, bottleneck dimension $d_b{=}512$,
$K{=}32$ learnable query tokens, $M{=}8$ specialized experts,
top-$k$ with $k{=}2$, $8$-head cross-attention inside each adapter,
and a router gating hidden size of~$512$.
Both the teacher and the student backbones are kept frozen throughout
training; the \emph{only} trainable parameters live in the {CKT module}.

\paragraph{Single adapter.}
Each adapter branch consists of a shared projection trunk
(\texttt{down\_proj}/\texttt{up\_proj} with GELU, LayerNorm and Dropout), a set of learnable query tokens, a multi-head cross-attention
block, and a post-attention LayerNorm.  Its parameter breakdown is
given in Table~\ref{tab:single-adapter}.

\begin{table}[h]
\centering
\caption{Parameter breakdown of a single adapter branch.
Dimensions: $d_{\text{tea}}{=}5120$, $d_{\text{stu}}{=}2048$,
$d_b{=}512$, $K{=}32$, cross-attention with $8$ heads.}
\label{tab:single-adapter}
\small
\begin{tabular}{lll r}
\toprule
Module & Shape & Role & \# Params \\
\midrule
\texttt{down\_proj}      & $\mathrm{Linear}(5120{\to}512)$    & teacher\,$\to$\,bottleneck     & $2{,}621{,}952$  \\
\texttt{up\_proj}        & $\mathrm{Linear}(512{\to}2048)$    & bottleneck\,$\to$\,student     & $1{,}050{,}624$  \\
\texttt{layer\_norm}     & $\mathrm{LayerNorm}(2048)$         & post-MLP norm                  & $4{,}096$        \\
\texttt{query\_tokens}   & $(1,32,2048)$                      & learnable queries              & $65{,}536$       \\
\texttt{cross\_attn}     & $\mathrm{MHA}(d{=}2048,h{=}8)$     & compression to $K$ tokens      & $16{,}785{,}408$ \\
\texttt{post\_attn\_norm}& $\mathrm{LayerNorm}(2048)$         & post-attn norm                 & $4{,}096$        \\
\midrule
\multicolumn{3}{l}{\textbf{Single adapter total}} & $\mathbf{20{,}531{,}712}$ \\
\bottomrule
\end{tabular}
\end{table}

Observation: the $Q/K/V$ and output projections of the cross-attention
block alone account for $\sim\!16.8$M parameters, i.e.\ roughly
$82\%$ of a single adapter's budget.  This is expected, since the
compressor operates in the student's feature space
($d_{\text{stu}}{=}2048$) and uses four $2048{\times}2048$ projection
matrices.

\paragraph{Full \textsc{AdapterBank}.}
The bank is composed of one always-on generalized adapter, $M{=}8$
specialized adapters that are sparsely activated by the router, and a
small two-layer gating network.  Its overall parameter count is
summarized in Table~\ref{tab:adapter-bank}.

\begin{table}[h]
\centering
\caption{Parameter breakdown of the full \textsc{AdapterBank} used in
our main LIBERO experiments.  The router is a two-layer MLP
($5120{\to}512{\to}8$) plus a single learnable noise scalar.}
\label{tab:adapter-bank}
\small
\begin{tabular}{l r r r}
\toprule
Component & \# Instances & \# Params per inst. & \# Params (total) \\
\midrule
Generalized adapter                      & $1$ & $20{,}531{,}712$ & $20{,}531{,}712$  \\
Specialized adapters                     & $8$ & $20{,}531{,}712$ & $164{,}253{,}696$ \\
Dynamic router ($5120{\to}512{\to}8$, + noise) & $1$ & $2{,}626{,}057$ & $2{,}626{,}057$  \\
\midrule
\textbf{Total trainable (\textsc{AdapterBank})} & & & $\mathbf{187{,}411{,}465}$ \\
\multicolumn{3}{l}{\emph{(i.e.\ $\approx\!187.4$M trainable parameters)}} & \\
\bottomrule
\end{tabular}
\end{table}

Specialized adapters account for roughly $88\%$ of all trainable
parameters ($164.3$M of $187.4$M), whereas the router adds only
$2.6$M.  Note that although the router selects only $k{=}2$
specialized experts per instance, top-$k$ routing affects
\emph{output composition}, not parameter count: all $M$ adapters store
their own weights and are therefore included in the total.

\paragraph{Cost relative to the frozen backbones.}
Table~\ref{tab:relative-cost} situates the CKT budget against the two
frozen WAMs.  Even with the full $M{=}8$ expert configuration, the
overhead remains below $1.2\%$ of the combined teacher--student
backbone size, justifying our characterization of CKT module as
parameter efficient.

\begin{table}[h]
\centering
\caption{CKT trainable parameter overhead relative to frozen WAM
backbones.  All percentages are with respect to the approximate
parameter counts of Cosmos-Policy--2B (student) and DreamZero--14B
(teacher).}
\label{tab:relative-cost}
\small
\begin{tabular}{l r r}
\toprule
Frozen reference                & Approx.\ size & CKT overhead \\
\midrule
Cosmos-Policy--2B (student)     & $\sim\!2.0$B  & $9.37\%$ \\
DreamZero--14B (teacher)        & $\sim\!14.0$B & $1.34\%$ \\
Teacher $+$ Student (combined)  & $\sim\!16.0$B & $1.17\%$ \\
\bottomrule
\end{tabular}
\end{table}

\paragraph{Scaling considerations.}
Two architectural knobs dominate the CKT budget:
(i)~the number of specialized experts $M$, which scales the bank
\emph{linearly} and accounts for the $164.3$M specialized-adapter
term; and (ii)~the per-adapter cross-attention block, which accounts
for $\sim\!82\%$ of a single branch.  Consequently:
\begin{itemize}
  \item Halving the number of specialized experts from $M = 8$ to
        $M = 4$ reduces the bank to $\approx 105$M trainable
        parameters; setting $M = 2$ yields $\approx 64$M.
  \item Further reductions can be obtained by lowering the number of
        attention heads or by sharing the cross-attention projection
        across experts, at a potential cost to expert specialization.
  \item The bottleneck dimension $d_b$ and the number of query tokens
        $K$ have a comparatively minor effect ($\sim\!3.6$M per
        adapter combined), so they can be tuned freely without
        strongly affecting the overall budget.
\end{itemize}
Unless stated otherwise, all ablations in the main paper share the
default configuration reported here.

\subsection{Details of Student-Side Context Injection}
\label{app:student_injection}

This appendix expands Sec.~\ref{sec:student_injection}: we (i)~specify
exactly how the augmented conditioning sequence $\tilde{E}_t$ enters the
frozen student DiT, (ii)~prove that the 3D rotary position encoding of every
future-frame visual token is invariant to the injection, (iii)~describe the
forward-hook implementation used in our codebase, and (iv)~contrast our
cross-attention K/V injection with the alternative of splicing $C_A$ into
the visual self-attention stream.

\subsubsection{How $\tilde{E}_t$ shapes future-frame visual tokens}

\label{app:cross_attn_pathway}

Each block $\ell \in \{1,\dots,L^{\mathrm{stu}}\}$ of the student DiT
implements the standard ``self-attention $\to$ cross-attention $\to$ FFN''
pattern, gated by an AdaLN modulation $e^{(\ell)}$ derived from the
diffusion timestep $\tau$:
\begin{align}
Z^{(\ell+\tfrac{1}{3})}
&= Z^{(\ell)} + \mathrm{SelfAttn}^{(\ell)}\!\bigl(\mathrm{AdaLN}_1(Z^{(\ell)},\,e^{(\ell)})\bigr),
\label{eq:self_attn_block}
\\[2pt]
Z^{(\ell+\tfrac{2}{3})}
&= Z^{(\ell+\tfrac{1}{3})} + \mathrm{CrossAttn}^{(\ell)}\!\Bigl(\,
   Q\!=\!\mathrm{AdaLN}_3\!\bigl(Z^{(\ell+\tfrac{1}{3})}\bigr),\;
   K\!=\!V\!=\!\tilde{E}_t\,\Bigr),
\label{eq:cross_attn_block}
\\[2pt]
Z^{(\ell+1)}
&= Z^{(\ell+\tfrac{2}{3})} + \mathrm{FFN}^{(\ell)}\!\bigl(\mathrm{AdaLN}_2(Z^{(\ell+\tfrac{2}{3})},\,e^{(\ell)})\bigr),
\label{eq:ffn_block}
\end{align}
where, importantly, our injection only modifies the K/V argument of the
cross-attention in (\ref{eq:cross_attn_block}); all attention heads,
projection matrices and AdaLN parameters remain frozen at their pretrained
values.

Expanding (\ref{eq:cross_attn_block}) per head $h \in \{1,\dots,H_{\mathrm{a}}\}$,
the residual delivered to a future-frame patch $z_{(t,h_p,w_p)}$ is
\begin{equation}
\Delta z_{(t,h_p,w_p)}^{(\ell)}
\;=\;
\sum_{h=1}^{H_{\mathrm{a}}} W_O^{(\ell,h)}
\sum_{j=1}^{\tilde{L}}
\alpha_{(t,h_p,w_p),\,j}^{(\ell,h)}\,
\bigl(\tilde{E}_t\,W_V^{(\ell,h)}\bigr)_j,
\label{eq:patch_update_appendix}
\end{equation}
with the attention weights
\begin{equation}
\alpha_{(t,h_p,w_p),\,j}^{(\ell,h)}
\;\propto\;
\exp\!\Bigl(
  \tfrac{1}{\sqrt{d_h}}\,
  \bigl\langle\,
    z_{(t,h_p,w_p)} W_Q^{(\ell,h)},\;
    \tilde{E}_t W_K^{(\ell,h)}\bigr|_j
  \bigr\rangle
\Bigr).
\label{eq:attn_weights_appendix}
\end{equation}

Two observations follow directly from~(\ref{eq:patch_update_appendix})--(\ref{eq:attn_weights_appendix}):
\begin{enumerate}[label=(\roman*),leftmargin=1.5em,itemsep=2pt,topsep=2pt]
  \item Information flow is strictly one-directional: $C_A$ contributes keys
        and values but never queries, so the cost of the cross-attention
        is $\mathcal{O}(L_v\,\tilde{L}\,d^{\mathrm{stu}})$ rather than the
        $\mathcal{O}(L_v^2\,d^{\mathrm{stu}})$ it would incur if $C_A$ were
        spliced into the self-attention stream.
  \item Each patch admits a personalized mixture of teacher tokens, indexed
        by its spatio-temporal coordinate; this is the mechanism by which
        a single observation-level $C_A$ can deliver location-specific
        cues to spatially distinct regions of the predicted future frames.
\end{enumerate}

\subsubsection{Position-encoding invariance}
\label{app:rope_invariance}

The student's visual tokens are equipped with a \emph{factorized 3D rotary
position encoding} (3D RoPE) applied inside the self-attention of
(\ref{eq:self_attn_block}). For a patch at index $(t,h_p,w_p)$ and head
dimension $d_h$, write
\begin{equation}
\Phi_{(t,h_p,w_p)}
\;=\;
\bigl[\,\phi_T(t)\,;\,\phi_H(h_p)\,;\,\phi_W(w_p)\,\bigr]
\;\in\; \mathbb{R}^{d_h},
\label{eq:rope_def}
\end{equation}
where $\phi_T,\phi_H,\phi_W$ are the temporal, height- and width-axis
rotation phases, allocated to disjoint slices of the head dimension. The
rotary basis is materialised as
\begin{equation}
\bigl(q^{\mathrm{rope}},\,k^{\mathrm{rope}}\bigr)_{(t,h_p,w_p)}
\;=\;
\bigl(R\bigl(\Phi_{(t,h_p,w_p)}\bigr)\,q_{(t,h_p,w_p)},\;
       R\bigl(\Phi_{(t,h_p,w_p)}\bigr)\,k_{(t,h_p,w_p)}\bigr),
\label{eq:rope_apply}
\end{equation}
with $R(\cdot)$ the standard block-diagonal rotation matrix.

\begin{proposition}[Position-encoding invariance under context injection]
\label{prop:rope_invariance}
For every block $\ell$, every patch index $(t,h_p,w_p)\in[T]\times[H]\times[W]$,
every denoising step $\tau$, and every choice of $C_A$ and
$\langle\textsc{sep}\rangle$, the rotary basis of (\ref{eq:rope_apply})
applied to that patch is unchanged when $E_t$ is replaced by $\tilde{E}_t$
in the cross-attention. Equivalently, neither the temporal index $t$, nor
the spatial indices $(h_p,w_p)$, nor the rotation matrices
$R(\Phi_{(t,h_p,w_p)})$ depend on $\tilde{L}$.
\end{proposition}

\begin{proof}
The proposition follows from three properties of the student backbone, as detailed in App.~\ref{app:implementation}.
\begin{enumerate}[label=(P\arabic*),leftmargin=2em,itemsep=2pt,topsep=2pt]
  \item \emph{RoPE indexing is geometry-bound, not sequence-bound.}
  The phases $\phi_T(t),\phi_H(h_p),\phi_W(w_p)$ in (\ref{eq:rope_def})
  are computed by enumerating the patch grid $(T,H,W)$ alone; they do not
  depend on the conditioning length $\tilde{L}$ or on any token in
  $\tilde{E}_t$. Hence $\Phi_{(t,h_p,w_p)}$ is identical under $E_t$ and
  $\tilde{E}_t$.

  \item \emph{RoPE is applied only inside visual self-attention.}
  In (\ref{eq:self_attn_block}), the queries and keys of the
  self-attention come from $Z^{(\ell)}$ exclusively; $\tilde{E}_t$ is not
  routed into self-attention. Therefore (\ref{eq:rope_apply}) is computed
  on tensors that are independent of $\tilde{E}_t$.

  \item \emph{Cross-attention applies no positional encoding on either side.}
  In (\ref{eq:cross_attn_block}), the cross-attention used by the student
  WAM applies no positional encoding to its queries (visual) or its keys
  and values (textual + transferred context). Consequently, the absolute
  position of $C_A$ inside $\tilde{E}_t$, and the presence of
  $\langle\textsc{sep}\rangle$, contribute only through their value
  embeddings via (\ref{eq:patch_update_appendix})--(\ref{eq:attn_weights_appendix})
  and not through any positional kernel.
\end{enumerate}
Combining (P1)--(P3) yields the proposition.
\end{proof}

\paragraph{Corollary: order-invariance of the conditioning sequence.}
By (P3), the cross-attention output in (\ref{eq:cross_attn_block}) is
invariant under any permutation of the rows of $\tilde{E}_t$. The exact
ordering chosen in (\ref{eq:inject_sparse}) -- $E_t$ first, then
$\langle\textsc{sep}\rangle$, then $C_A$ -- is therefore a stylistic choice
and does not affect the network's outputs at initialization or after
training.\footnote{In practice we still keep a fixed ordering at training
and evaluation so that the student's K/V cache layout is deterministic.}

\subsubsection{Implementation as a forward hook}
\label{app:implementation}

Because the augmentation in (\ref{eq:inject_sparse}) acts only on the
output of the student's text-projection module, it is implemented as a
single non-intrusive forward hook on $g_\theta$:
\begin{equation}
g_\theta(\cdot) \,\rightsquigarrow\, \tilde{g}_\theta(\cdot) \;:=\;
\bigl[\, g_\theta(\cdot) \,;\, \langle\textsc{sep}\rangle \,;\, C_A \,\bigr],
\end{equation}
registered before each forward pass and cleared afterwards. This means in
particular that:
\begin{itemize}[leftmargin=1.2em,itemsep=2pt,topsep=2pt]
  \item \textbf{No edits to the student source code.} The DiT module file
  is loaded verbatim from the pretrained checkpoint; the injection lives
  entirely in our pipeline wrapper.
  \item \textbf{Trainable parameters are confined to the CKT module.}
  The hook returns a concatenation, which has no learnable parameters of
  its own (apart from the optional $\langle\textsc{sep}\rangle$ token);
  gradients flow only through $C_A$ back into the CKT adapter bank and
  router.
  \item \textbf{Amortization across denoising steps.} The teacher and
  the CKT module are evaluated once per observation, the resulting
  $C_A$ is cached on the wrapper, and every subsequent call of the
  student's $g_\theta$ -- one per denoising step per CFG branch -- reads
  the same cached $C_A$. The marginal compute relative to the
  unaugmented student is one extra K/V projection over $\tilde{L}-L_t-1$
  tokens per cross-attention layer.
\end{itemize}

\subsubsection{Why not splice $C_A$ into the visual stream?}
\label{app:why_not_visual}

A natural alternative is to concatenate $C_A$ with the visual tokens in the
self-attention stream of (\ref{eq:self_attn_block}) instead of with the
conditioning tokens. We deliberately avoid this design for three reasons.
\begin{enumerate}[leftmargin=1.5em,itemsep=2pt,topsep=2pt]
  \item \textbf{It breaks the rotary basis.} The 3D RoPE in
  (\ref{eq:rope_def})--(\ref{eq:rope_apply}) requires the visual sequence
  length to equal $T H W$ exactly; appending $K_g + K_s$ extra tokens to
  $Z_\tau$ violates this invariant. Restoring it would require ad-hoc
  positional indices for the spliced tokens, all of which interact with
  the pretrained rotary basis in unpredictable ways and were observed in
  preliminary experiments to destabilize denoising.
  \item \textbf{It blows up self-attention cost.} Self-attention is
  $\mathcal{O}((L_v + 2K)^2 d^{\mathrm{stu}})$ instead of
  $\mathcal{O}(L_v^2 d^{\mathrm{stu}})$, and unlike cross-attention the
  cost cannot be hidden behind FlashAttention's K/V-asymmetric kernel.
  \item \textbf{It conflates roles.} Self-attention is the student's
  pretrained spatial-temporal mixer over visual content; pushing teacher
  tokens through this pathway would force them to behave as pseudo-patches
  competing with future-frame patches for the same attention budget.
  Cross-attention, in contrast, is the student's native conditioning
  pathway, and is exactly the role $C_A$ is meant to play.
\end{enumerate}
The injection in (\ref{eq:inject_sparse}) thus realizes the smallest
possible interface change to the frozen student that still grants every
future-frame visual token access to the teacher's transferred knowledge,
while preserving its pretrained geometry, attention budget, and inference
amortization.

\subsection{Training Dynamics of the Balancing Objective}
\label{app:training_dynamics_balancing}

\begin{figure}[t]
    \centering
    \includegraphics[width=\textwidth]{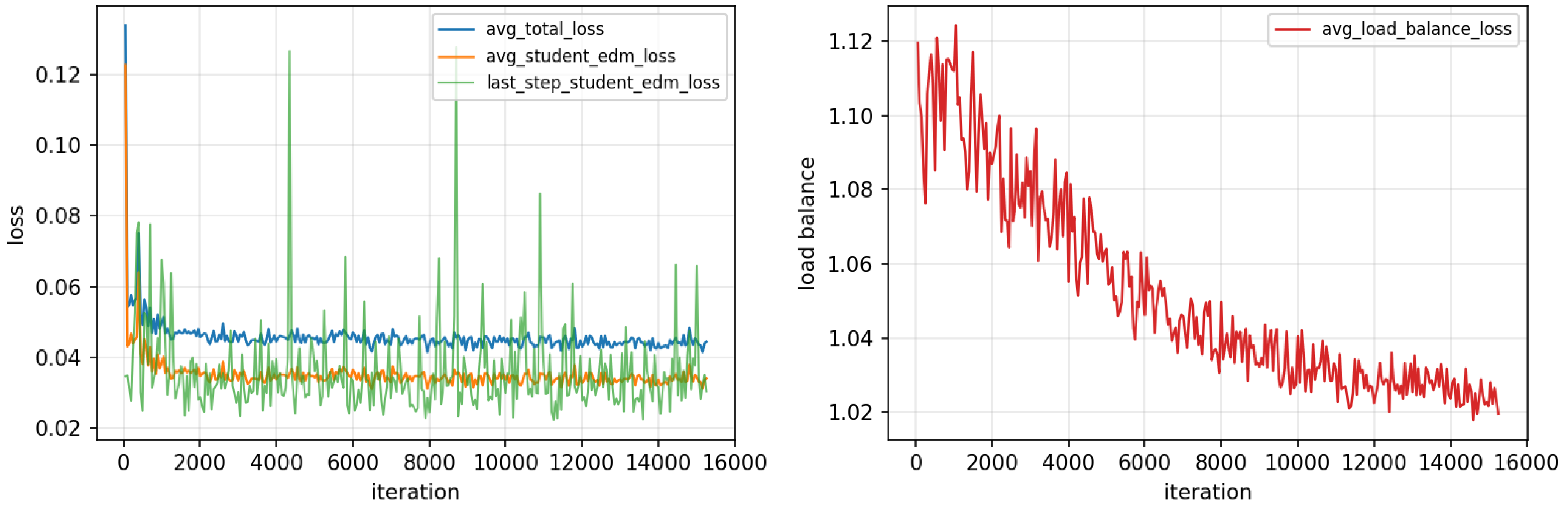}
    \caption{
    Training curves of the proposed method.
    \textbf{Left:} optimization trajectories of the total loss, the averaged student EDM loss, and the last-step student EDM loss.
    \textbf{Right:} trajectory of the load-balancing loss $\mathcal{L}_{\mathrm{bal}}$.
    The balancing term decreases steadily throughout training, while the overall objective remains stable.
    }
    \label{fig:training_dynamics_balancing}
\end{figure}

Figure~\ref{fig:training_dynamics_balancing} shows the training dynamics of the proposed objective.
Recall that the final loss is defined as
$\mathcal{L} = \mathcal{L}_{\mathrm{CKT}} + \lambda_{\mathrm{bal}} \mathcal{L}_{\mathrm{bal}}$,
where
\begin{equation}
\mathcal{L}_{\mathrm{bal}}
=
M \sum_{m=1}^{M} f_m P_m,
\end{equation}
with $\lambda_{\mathrm{bal}}=0.01$. 
This term penalizes routing imbalance by encouraging the marginal routing probabilities and actual selection frequencies to be more uniformly distributed across adapters.

Several observations can be made.
First, the optimization remains stable throughout training.
As shown in the left panel, both the total loss and the averaged student EDM loss drop rapidly in the early stage and then enter a stable plateau, indicating that the auxiliary balancing term does not destabilize the total loss.
The last-step student EDM loss exhibits visibly larger variance, which is expected since it is measured from a single denoising step and is therefore more sensitive to minibatch variation and timestep-dependent difficulty.
Importantly, despite such local fluctuations, its overall scale remains bounded and does not show progressive drift, suggesting that training is well behaved.

Second, the balancing objective is consistently optimized and improves routing utilization over time.
In the right panel, $\mathcal{L}_{\mathrm{bal}}$ steadily decreases from roughly $1.12$ at the beginning of training to about $1.02$ near convergence.
This trend indicates that adapter usage becomes increasingly balanced as training proceeds, rather than collapsing to a small subset of adapters.
Notably, under perfectly balanced routing, one would have $P_m = 1/M$ and $f_m = 1/M$ for all $m$, which yields
\begin{equation}
\mathcal{L}_{\mathrm{bal}} = 1.
\end{equation}
Therefore, the final value close to $1.0$ suggests that the learned router approaches a near-uniform operating regime while still preserving enough flexibility for input-dependent specialization.

Third, the gap between the total loss and the student EDM loss remains small and stable, which further indicates that the balancing regularizer acts as a mild auxiliary constraint rather than dominating optimization.
This behavior is desirable: the model continues to reduce the primary generation error, while the routing structure is simultaneously regularized toward healthier adapter utilization.
Taken together, these curves support that the proposed balancing formulation achieves its intended effect---it improves adapter allocation without introducing noticeable optimization conflict with the main training objective.

Overall, Figure~\ref{fig:training_dynamics_balancing} provides empirical evidence that the load-balancing term is both \emph{effective} and \emph{non-disruptive}: it steadily mitigates routing imbalance, drives the routing statistics toward the balanced regime, and preserves stable optimization of the student objective throughout training.

\subsection{Results on LIBERO Benchmark}


\begin{table}[t]
\centering
\caption{LIBERO simulation benchmark results. Success rates (SR, \%) across four LIBERO benchmark task suites \citep{libero}.}
\label{tab:libero_sim}
\begin{tabular}{lccccc}
\toprule
Method & Spatial & Object & Goal & Long & Average \\
\midrule
Diffusion Policy  & 78.3 & 92.5 & 68.3 & 50.5 & 72.4 \\
Dita  & 97.4 & 94.8 & 93.2 & 83.6 & 92.3 \\
$\pi_0$ & 96.8 & 98.8 & 95.8 & 85.2 & 94.2 \\
UniVLA & 96.5 & 96.8 & 95.6 & 92.0 & 95.2 \\
$\pi_{0.5}$ & 98.8 & 98.2 & 98.0 & 92.4 & 96.9 \\
OpenVLA-OFT & 97.6 & 98.4 & 97.9 & 94.5 & 97.1 \\
CogVLA & 98.6 & 98.8 & 96.6 & 95.4 & 97.4 \\
Cosmos Policy & 98.1 & \textbf{100.0} & 98.2 & 97.6 & 98.5 \\
\midrule
CKT-WAM (Ours) & \textbf{99.0} & 99.8 & \textbf{98.6} & \textbf{97.8} & \textbf{98.8} \\
\bottomrule
\end{tabular}
\end{table}

Table~\ref{tab:libero_sim} reports the results on the standard LIBERO simulation benchmark.
The student WAM is initialized from the official Cosmos-Policy-LIBERO-Predict2-2B checkpoint.
CKT-WAM achieves the best overall performance, reaching an average success rate of 98.8\%, which improves over the strongest prior baseline, Cosmos Policy, by 0.3 points. Although the benchmark is already highly saturated for top-performing methods, CKT-WAM still delivers consistent gains on three out of four suites, including Spatial, Goal, and Long, while remaining competitive on Object. In particular, the improvement on the Long suite suggests that the proposed cross-model knowledge transfer is especially beneficial for temporally extended manipulation tasks, where reliable action generation over longer horizons is more critical. Overall, these results show that CKT-WAM not only preserves the strong in-domain capabilities of modern VLA policies, but also provides a more effective policy representation that translates into stronger and more consistent performance across diverse manipulation scenarios.

\section{Detailed Implementation of the Load Balancing Auxiliary Loss}
\label{app:load_balancing}

To encourage balanced utilization of specialized adapters and avoid routing collapse, we introduce an auxiliary load-balancing objective following the standard practice in sparse mixture-of-experts (MoE) models. Without such a regularizer, the router may over-concentrate probability mass on a small subset of adapters, causing under-utilization of the remaining capacity and limiting the effectiveness of specialization.

Let $p_{b,m}$ denote the routing probability assigned to the $m$-th adapter for the $b$-th instance in a mini-batch of size $B$, where $m \in \{1,\dots,M\}$ and $M$ is the total number of specialized adapters. We first define the batch-averaged routing mass of adapter $m$ as
\begin{equation}
P_m = \frac{1}{B}\sum_{b=1}^{B} p_{b,m},
\label{eq:Pm_appendix}
\end{equation}
which measures how much routing probability, on average, is allocated to adapter $m$ over the batch.

Since our routing is implemented in a sparse top-$k$ manner, each instance activates only a small subset of adapters. For the $b$-th instance, we denote by $\mathcal{I}_b = \operatorname{Top\text{-}K}(\{p_{b,m}\}_{m=1}^{M}, k)$ the index set of the $k$ adapters with the largest routing probabilities. Based on this selection set, we define the empirical top-$k$ selection frequency of adapter $m$ as
\begin{equation}
f_m = \frac{1}{kB}\sum_{b=1}^{B} \mathbb{I}\{m \in \mathcal{I}_b\},
\label{eq:fm_appendix}
\end{equation}
where $\mathbb{I}\{\cdot\}$ is the indicator function. Intuitively, $f_m$ reflects how often adapter $m$ is actually selected for execution, while $P_m$ captures how much routing probability mass it receives before hard top-$k$ selection.

Following sparse MoE formulations, we define the load-balancing loss as
\begin{equation}
\mathcal{L}_{\mathrm{bal}} = M \sum_{m=1}^{M} f_m P_m.
\label{eq:lbal_appendix}
\end{equation}
This objective is minimized when routing mass and selection frequency are distributed more evenly across adapters. In other words, it discourages degenerate solutions in which a few adapters dominate both the router probabilities and the actual top-$k$ assignments. The multiplicative factor $M$ is introduced only for scale normalization, so that the magnitude of the auxiliary term remains stable as the number of adapters varies.
The final training objective is
\begin{equation}
\mathcal{L} = \mathcal{L}_{\mathrm{CKT}} + \lambda_{\mathrm{bal}} \mathcal{L}_{\mathrm{bal}},
\label{eq:loss_total_appendix}
\end{equation}
where $\mathcal{L}_{\mathrm{CKT}}$ is the main training objective and $\lambda_{\mathrm{bal}}$ controls the strength of the auxiliary regularization. In practice, this term improves adapter utilization and stabilizes specialization by preventing the router from collapsing to a narrow subset of experts, while still allowing different adapters to develop distinct functional roles.

\section{Detailed Related Work}
\label{related_work2}

\subsection{World Action Models}
Early world-model research focused on learning compact latent dynamics for prediction, planning, and imagination-based control \citep{ha2018worldmodels, hafner2019planet, hafner2020dreamer, hafner2023dreamerv3}.

Building upon these foundations, recent literature has formally established WAMs as a powerful alternative for embodied control, explicitly modeling how visual observations may evolve under action 
\citep{ye2026world}. Within this rapidly evolving landscape, two dominant generative paradigms have emerged. 
The first is the \emph{imagine-then-execute} (or causal) approach, where models such as LingBot-VA \citep{li2026causal} and Vidar \citep{feng2025vidar} first generate future visual trajectories and subsequently condition action prediction on these imagined futures. 
The second paradigm involves \emph{joint modeling}, where future video and action tokens are denoised together within a shared generative architecture, as demonstrated by Motus \citep{bi2025motus} and other unified frameworks
\citep{zhu2025unified}.

Standard VLA models, while successful in scaling reactive visual-motor policies \citep{brohan2023rt1, zitkovich2023rt2, kim2025openvla}, largely rely on static image-text pretraining and lack this explicit physical dynamics modeling. 
However, the explicit video synthesis utilized by most WAMs incurs substantial test-time latency due to iterative denoising. 
To mitigate this efficiency bottleneck, parallel efforts such as VPP \citep{hu2024video} and UVA \citep{li2025unified} have begun exploring ways to utilize predictive representations while reducing or entirely bypassing explicit video decoding during inference.

Our work is complementary to these architectural advances. Instead of proposing a new monolithic WAM, we study how action-relevant dynamics knowledge can be efficiently transferred \emph{between} WAMs. 
In particular, CKT-WAM focuses on a practical teacher--student setting, where a stronger teacher WAM provides compact intermediate context to a more lightweight and efficient student WAM.

\subsection{Knowledge Transfer}
Knowledge transfer in neural networks has evolved from output-space teacher--student learning \citep{hinton2015distilling} to richer forms of intermediate supervision. 
FitNets introduced the idea of transferring hidden-layer hints \citep{romero2015fitnets}, while attention transfer showed that structured attention maps can serve as an effective transfer target \citep{zagoruyko2017attention}. 
In reinforcement learning, policy distillation demonstrated that action policies can be transferred from a teacher to a student network \citep{rusu2016policy}. 
For Transformer-based models, works such as TinyBERT further explored multi-level transfer over hidden states and attention patterns \citep{jiao2020tinybert}.

Despite this progress, most prior transfer methods operate by matching teacher and student outputs, hidden representations, or attention statistics directly. 
Such formulations are often most natural when the teacher and student share closely aligned prediction spaces or architectures. 
By contrast, WAM-to-WAM knowledge transfer introduces additional challenges, since different models may use different latent parameterizations, token interfaces, or denoising backbones. 

CKT-WAM therefore adopts a different perspective. 
Rather than treating transfer as direct output imitation or full representation matching, we cast it as \emph{contextual knowledge transfer}: the teacher provides a compact set of transferable context tokens extracted from an intermediate block, and the student consumes these tokens through its native conditioning pathway. 
This design preserves architectural flexibility, reduces transfer overhead, and is naturally suited to heterogeneous teacher--student WAM pairs.

\subsection{Parameter Efficient Adapters}
Parameter-efficient fine-tuning (PEFT) has emerged as a practical alternative to full fine-tuning when adapting large pretrained backbones to downstream tasks. A representative line is low-rank reparameterization, where LoRA models task-specific updates in a low-dimensional subspace with substantially reduced trainable parameters \citep{hu2022lora}. Building on this idea, recent spectral variants such as PiSSA and KaSA show that the choice of singular components and initialization can substantially affect optimization efficiency and knowledge preservation \citep{meng2024pissa,wang2025kasa}. Orthogonally, MoE-style PEFT methods further increase adaptation capacity through routed low-rank experts without reverting to full parameter updates, as exemplified by AdaMoLE and HydraLoRA \citep{liu2024adamole,tian2024hydralora}. In vision, however, effective adaptation depends not only on parameter count but also on whether the adapter design matches visual inductive biases: AdaptFormer demonstrates that lightweight bottleneck adapters can effectively adapt frozen vision transformers across image and video recognition tasks \citep{chen2022adaptformer}, ViT-Adapter injects task-relevant visual priors to make plain ViTs competitive for dense prediction \citep{chen2023vitadapter}, and the recent Mona argues that visual adaptation should move beyond NLP-style linear adapters by introducing multi-cognitive visual filters, showing that lightweight adapter tuning can even outperform full fine-tuning on several visual recognition benchmarks \citep{yin2025mona}. Our method is related to this line of work, but differs in emphasizing a structured parameter-efficient adaptation design tailored to balancing general knowledge retention and task-specific specialization.


\newpage
\input{checklist.tex}

\end{document}

%% file: checklist.tex
\section*{NeurIPS Paper Checklist}
\begin{enumerate}

\item {\bf Claims}
    \item[] Question: Do the main claims made in the abstract and introduction accurately reflect the paper's contributions and scope?
    \item[] Answer: \answerYes{} 
    \item[] Justification: The main claims made in the abstract and introduction accurately reflect the paper's contributions and scope.

\item {\bf Limitations}
    \item[] Question: Does the paper discuss the limitations of the work performed by the authors?
    \item[] Answer: \answerYes{} 
    \item[] Justification: Limitations are discussed in the 5th section.

\item {\bf Theory assumptions and proofs}
    \item[] Question: For each theoretical result, does the paper provide the full set of assumptions and a complete (and correct) proof?
    \item[] Answer: \answerYes{} 
    \item[] Justification: Appendix contains Proposition 1 and its corresponding formal proof regarding position-encoding invariance.

    \item {\bf Experimental result reproducibility}
    \item[] Question: Does the paper fully disclose all the information needed to reproduce the main experimental results of the paper to the extent that it affects the main claims and/or conclusions of the paper (regardless of whether the code and data are provided or not)?
    \item[] Answer: \answerYes{} 
    \item[] Justification: All the information needed to reproduce the main experimental results of the paper is shown in the 3rd and 4th sections.

\item {\bf Open access to data and code}
    \item[] Question: Does the paper provide open access to the data and code, with sufficient instructions to faithfully reproduce the main experimental results, as described in supplemental material?
    \item[] Answer: \answerYes{} 
    \item[] Justification: We provide open access to the data and code in the Reproducibility Statement Section

\item {\bf Experimental setting/details}
    \item[] Question: Does the paper specify all the training and test details (e.g., data splits, hyperparameters, how they were chosen, type of optimizer) necessary to understand the results?
    \item[] Answer: \answerYes{} 
    \item[] Justification: All the training and test details are shown in the Experiments section.

\item {\bf Experiment statistical significance}
    \item[] Question: Does the paper report error bars suitably and correctly defined or other appropriate information about the statistical significance of the experiments?
    \item[] Answer: \answerYes{} 
    \item[] Justification: We report the number of trials in real-world experiments.

\item {\bf Experiments compute resources}
    \item[] Question: For each experiment, does the paper provide sufficient information on the computer resources (type of compute workers, memory, time of execution) needed to reproduce the experiments?
    \item[] Answer: \answerYes{} 
    \item[] Justification: The paper provides sufficient information on the computer resources in the Experiments Section.
    
\item {\bf Code of ethics}
    \item[] Question: Does the research conducted in the paper conform, in every respect, with the NeurIPS Code of Ethics \url{https://neurips.cc/public/EthicsGuidelines}?
    \item[] Answer: \answerYes{} 
    \item[] Justification: The research conducted in the paper conform, in every respect, with the NeurIPS Code of Ethics.

\item {\bf Broader impacts}
    \item[] Question: Does the paper discuss both potential positive societal impacts and negative societal impacts of the work performed?
    \item[] Answer: \answerYes{} 
    \item[] Justification: the paper discusses both potential positive societal impacts and negative societal impacts of the work performed.
    
\item {\bf Safeguards}
    \item[] Question: Does the paper describe safeguards that have been put in place for responsible release of data or models that have a high risk for misuse (e.g., pre-trained language models, image generators, or scraped datasets)?
    \item[] Answer: \answerNA{} 
    \item[] Justification: the paper poses no such risks.

\item {\bf Licenses for existing assets}
    \item[] Question: Are the creators or original owners of assets (e.g., code, data, models), used in the paper, properly credited and are the license and terms of use explicitly mentioned and properly respected?
    \item[] Answer: \answerYes{} 
    \item[] Justification: the creators and original owners of assets are properly credited, and the license and terms of use are explicitly mentioned and are properly respected.

\item {\bf New assets}
    \item[] Question: Are new assets introduced in the paper well documented and is the documentation provided alongside the assets?
    \item[] Answer: \answerYes{} 
    \item[] Justification: Codes introduced in the paper are well documented in the anonymous code link.

\item {\bf Crowdsourcing and research with human subjects}
    \item[] Question: For crowdsourcing experiments and research with human subjects, does the paper include the full text of instructions given to participants and screenshots, if applicable, as well as details about compensation (if any)? 
    \item[] Answer: \answerNA{} 
    \item[] Justification: the paper does not involve crowdsourcing nor research with human subjects.

\item {\bf Institutional review board (IRB) approvals or equivalent for research with human subjects}
    \item[] Question: Does the paper describe potential risks incurred by study participants, whether such risks were disclosed to the subjects, and whether Institutional Review Board (IRB) approvals (or an equivalent approval/review based on the requirements of your country or institution) were obtained?
    \item[] Answer: \answerNA{} 
    \item[] Justification: the paper does not involve crowdsourcing nor research with human subjects.

\item {\bf Declaration of LLM usage}
    \item[] Question: Does the paper describe the usage of LLMs if it is an important, original, or non-standard component of the core methods in this research? Note that if the LLM is used only for writing, editing, or formatting purposes and does \emph{not} impact the core methodology, scientific rigor, or originality of the research, declaration is not required.
    \item[] Answer: \answerYes{} 
    \item[] Justification: the paper describes the usage of LLMs in the appendix.

\end{enumerate}